%% file: main.tex
\documentclass[12pt]{article}

\usepackage{acro}
\usepackage{amsmath}
\usepackage{amssymb}
\usepackage{amsthm}
\usepackage{bm}
\usepackage{enumitem}
\usepackage{fullpage}
\usepackage{hyperref}
\usepackage{natbib}
\usepackage{xcolor}
\usepackage{graphicx} 
\usepackage{url}
\usepackage{booktabs}
\usepackage{cleveref}
\usepackage{multirow}
\usepackage{subcaption}
\usepackage{sidecap, caption}
\usepackage{booktabs}
\usepackage{colortbl}
\usepackage{placeins}

\input{macros}

\begin{document}

\title{A Distributional Optimisation Perspective on Combining Models in Deep Learning}
\author{Congye Wang$^{\star}$, Yan Lin$^{\star}$, Zheyang Shen, \\
Matthew A. Fisher, Chris. J. Oates \\
\small $^\star$ equal contribution \\
\small Newcastle University, UK 
}

\maketitle

\begin{abstract}
\input{abstract}
\end{abstract}

\input{body}

\paragraph{Acknowledgements}
\input{acks}

\appendix
\input{appendix}
\FloatBarrier
\newpage

\bibliographystyle{abbrvnat}
\bibliography{bibliography}

\end{document}

%% file: macros.tex
\DeclareMathOperator*{\argmin}{arg\,min}

\DeclareAcronym{KL}{short = KL, long = Kullback--Leibler}
\DeclareAcronym{LA}{short = LA, long = LoRA averaging}
\DeclareAcronym{LLM}{short = LLM, long = large language model}
\DeclareAcronym{LoRA}{short = LoRA, long = low-rank adapter}
\DeclareAcronym{MFLD}{short = MFLD, long = mean-field Langevin dynamics}
\DeclareAcronym{MFNN}{short = MFNN, long = mean-field neural network}
\DeclareAcronym{MoE}{short = MoE, long = mixtures of experts}
\DeclareAcronym{OA}{short = OA, long = output averaging}
\DeclareAcronym{ODE}{short = ODE, long = ordinary differential equation}
\DeclareAcronym{PA}{short = PA, long = parameter averaging}
\DeclareAcronym{RKHS}{short = RKHS, long = reproducing kernel Hilbert space}
\DeclareAcronym{VGD}{short = VGD, long = variational gradient descent}
\DeclareAcronym{FVGD}{short = FVGD, long = functional variational gradient descent}
\DeclareAcronym{MLP}{short = MLP, long = multi-layer perceptron}
\DeclareAcronym{PoC}{short = PoC, long = propagation of chaos}

\definecolor{yellowbox}{RGB}{255, 255, 0} 

\newtheorem*{proposition*}{Proposition}

\newtheorem{remark}{Remark}

%% file: abstract.tex
Combining predictions from different models can improve performance at machine learning tasks, but the training of the individual models and the rule used to combine them are typically chosen separately, and by \emph{ad hoc} means.
Recent advances in \emph{distributional optimisation} (i.e. where the optimisation occurs over the set of probability distributions) offer an opportunity for principled \emph{joint} training, viewing the collection of models as a discrete distribution whose support points are to be optimised, but the potential of these methods is not well-understood.
In this paper we (1) cast two standard combination strategies -- ensembles and low-rank adapter averaging -- as entropy-regularised distributional optimisation, observing that the resulting objective is convex in the ensemble case but not in the adapter-averaging case, so that existing convergence guarantees for \ac{MFLD} transfer only to the former;
(2) assess existing and novel algorithms for this task, including a functional variant of \ac{VGD};
and (3) report an empirical study spanning synthetic classification tasks and fine-tuning of large language models on a commonsense reasoning benchmark.

%% file: body.tex
\section{Introduction}
\label{sec: intro}

Ensemble methods and mixtures of experts aim to improve predictive performance, robustness, and uncertainty estimation by combining multiple models rather than relying on a single model.
Ensemble methods typically aggregate the outputs of independently trained models -- through averaging, voting, or stacking -- to reduce variance and mitigate overfitting, building on ideas from classical machine learning such as bagging and boosting \citep{Breiman1996,dietterich2000ensemble}. 
Mixtures of experts take a more structured approach, learning a set of specialised expert networks along with a gating mechanism that dynamically selects or weights experts based on the input \citep{Jacobs1991, Jordan1994}. 
This allows different experts to focus on distinct regions of the input space or subtasks, improving efficiency and expressiveness. 
Both approaches leverage model diversity to achieve better generalisation and have been widely applied in areas such as vision, natural language processing, and large-scale recommendation systems \citep{Lakshminarayanan2017, Shazeer2017}.

Despite their effectiveness, the training of individual models and their combination in ensemble methods and mixtures of experts is often guided by heuristic or \emph{ad hoc} design choices. 
Models may be trained independently with different random initialisations, architectures, or data subsets without a principled objective that explicitly accounts for their eventual combination \citep{dietterich2000ensemble}. 
Similarly, the choice of aggregation strategy -- such as simple averaging, fixed weighting, or a particular gating architecture -- is frequently motivated by empirical performance rather than theoretical guarantees. 
As a result, while these methods work well in practice, their design and optimisation can lack a unified, systematic framework, and the interaction between training procedures and combination rules remains an active area of research.
The aim of this paper is to explore whether distributional optimisation can provide such a framework.

\subsection{Combining Models in Deep Learning}
\label{sec: context}

Consider learning a function (or \emph{model}) $f$, capable of transforming inputs $\mathbf{x}$ into appropriate outputs $f(\mathbf{x})$ (e.g. logits for a classification task, or a point estimate for a regression task).
In this section we review two of the main strategies in which $f$ is constructed, in turn, from a discrete set of candidate models $\{f_i\}_{i=1}^m$.
Other strategies are discussed in \Cref{app: other approaches}.

\paragraph{Ensembles} 

Assuming each model $f_i$ produces output $f_i(\mathbf{x})$ in a common vector space, such as $\mathbb{R}^d$, one can construct an \emph{ensemble}  \citep{dietterich2000ensemble}
\begin{align}
f(\mathbf{x}) = \sum_{i=1}^m w_i f_i(\mathbf{x}) , \label{eq: output ave}
\end{align}
i.e. a weighted average of the outputs from each model.
Remarkably, even a simple uniform average (i.e. $w_i = \frac{1}{m}$) of models trained with different hyperparameters (e.g. learning rates) can empirically improve performance \citep{gontijono} and offer improved robustness to distribution shift \citep{ovadia2019can}.
However, a uniform average may be sub-optimal.
Accordingly, several strategies have been proposed for learning appropriate weights $w_i$, with a benefit from non-uniform weighting reported \citep{caruana2004ensemble,caruana2006getting,levesque2016bayesian,wenzel2020hyperparameter}.
A limiting instance of this approach is to use what is believed to be the single best-performing model.
The main limitation of ensembles is that \emph{ad hoc} strategies are typically used to determine the candidate model set $\{f_i\}_{i=1}^m$, so that in practice some models may receive little or no weight, meaning that the computational effort involved in training these models has been wasted.

\paragraph{LoRA Averaging}

\textit{Ab initio} training can incur a high computational cost; instead, techniques such as \emph{low-rank adapters} (\acs{LoRA}) are often used to fine-tune a foundational model \citep{hu2022lora}.
Recall that a deep neural network contains many dense layers, where in each layer the inputs are multiplied by a \emph{weight matrix} $\mathbf{W} \in \mathbb{R}^{d \times k}$.
The idea of \ac{LoRA} is to replace the weight matrix $\mathbf{W}$ of the reference agent by $\mathbf{W} + \mathbf{B} \mathbf{A}$ where $\mathbf{B} \in \mathbb{R}^{d \times r}$, $\mathbf{A} \in \mathbb{R}^{r \times k}$, and $r \ll \min\{d,k\}$. 
From this construction, $\mathbf{B} \mathbf{A}$ has a low rank.
The pair $(\mathbf{A},\mathbf{B})$ is called the \emph{adapter} for that layer, and the benefit of \ac{LoRA} is that each layer has only a (relatively) small number parameters that need to be learned during fine-tuning.
For simplicity, consider the case in which a single weight matrix $\mathbf{W}$ is fine-tuned. 
Supposing that we have a collection of fine-tuned models $f_i(\mathbf{x}) = f(\mathbf{x}, \mathbf{W} + \mathbf{B}_i \mathbf{A}_i)$, where $(\mathbf{A}_i,\mathbf{B}_i)$ is the adapter for the $i$th fine-tuned model, a simple approach to combining these models is \emph{LoRA averaging},
\begin{align}
f(\mathbf{x}) = f\left( \mathbf{x} , \mathbf{W} + \sum_{i=1}^m w_i \mathbf{B}_i \mathbf{A}_i \right) , \label{eq: LoRA ave}
\end{align}
where the adapter in \eqref{eq: LoRA ave} has rank at most $mr$, and the weights $w_i$ can potentially be optimised \citep{li2025efficient}.
Depending on the nature of the neural network, only a subset of the weight matrices may be adapted; for example, \citet{hu2022lora} considered the transformer architecture and adapted only the weight matrices in the self-attention module, leaving the weights in the multilayer perceptron unchanged.

\subsection{Our Contributions}

Despite formal hierarchical modelling principles being well-understood, in practice models are often independently trained before being combined.
In instances where models are jointly trained, such training is often \emph{ad hoc} due to the challenges associated with optimisation in higher-dimensional spaces when multiple models are considered.
This poses a barrier to methodological development and results in wasteful computation, since models may be trained whose contribution to the final prediction may be minimal.
However, a powerful idea that is relatively unexplored in this context is to `lift' the optimisation problem from the parameter space $\bm{\theta} \in \mathbb{R}^p$ to the space of probability distributions $\mathcal{P}(\mathbb{R}^p)$, in principle enabling an infinite number of potential models to be simultaneously considered.

Inspired by recent advances in \emph{distributional optimisation}, the aim of this work is to explore casting model combination as an optimisation task 
\begin{align}
    \argmin_{Q \in \mathcal{P}(\mathbb{R}^p)} \; \mathcal{J}(Q), \qquad \mathcal{J}(Q) = \mathcal{L}(Q) + \lambda \mathcal{E}(Q) , \label{eq: dist opt}
\end{align}
where $Q$ is a distribution over the candidate model (parameter) set, $\mathcal{L}(Q)$ is a \emph{loss function} capturing the performance of the combined model, $\lambda$ controls the amount of regularisation used, and $\mathcal{E}(Q)$ is the (negative) entropy of $Q$.
Our specific contributions are:
\begin{itemize}[leftmargin=*]
\itemsep0em 
    \item to formulate appropriate loss functions for ensembles and \ac{LoRA} averaging (cf. \Cref{sec: context}), so that the performance of the combined model can be explicitly optimised. 
    \item to explore algorithms suitable for entropy-regularised objectives such as \eqref{eq: dist opt}, including a novel \emph{functional} variant of variational gradient descent, for which we report a carefully diagnosed negative result;
    \item to empirically assess these methods' potential, including in challenging applications to \acp{LLM}.
\end{itemize}
It is important to emphasise that our aim is not to match the state-of-the-art; it is to objectively assess the potential of formalising model combination within a distributional optimisation framework.

\subsection{Related Work}

Our inspiration comes from \acp{MFNN}, a theoretical tool to understand the dynamics of gradient descent applied to networks with a single hidden layer \citep{nitanda2017stochastic,mei2018mean,chizat2018global}.
Let $\Phi$ be a single neuron, for example $\Phi(\mathbf{x} , \bm{\theta}) = \mathrm{ReLU}(\mathbf{W} \mathbf{x} + \mathbf{b})$ where the parameter $\bm{\theta} \in \mathbb{R}^p$ collects together weights $\mathbf{W}$ and biases $\mathbf{b}$.
The corresponding \ac{MFNN} $f(\mathbf{x}) = \int \Phi(\mathbf{x} , \bm{\theta}) \; \mathrm{d}Q(\bm{\theta})$ generalises from a finite number of neurons (when $Q$ has finite support) to the case where there are a possibly infinite number of neurons in a single hidden layer of the network.
In effect, the distribution $Q$ is the `parameter' of the \ac{MFNN}, and training can be conceptualised as optimisation over $Q \in \mathcal{P}(\mathbb{R}^p)$.
This insight enabled detailed theoretical analyses, such as \citet{nitanda2025propagation}, and also underpins the present work. 
However, the potential of distributional optimisation applied to \emph{deep} learning tasks remains poorly understood, motivating the present work.

\section{Methods}

After introducing our setting and notation in \Cref{subsec: setup}, we formulate model combination as a distributional optimisation task in \Cref{subsec: joint train} and discuss both existing and novel algorithms in \Cref{sec: algorithms}.

\subsection{Set-Up and Notation}
\label{subsec: setup}

\paragraph{Assumptions on the Model}
For this work we consider a model $f(\cdot , \bm{\theta}) : \mathbb{R}^d \rightarrow \mathbb{R}^e$ parametrised by $\bm{\theta} \in \mathbb{R}^p$.
It will be assumed that $f(\mathbf{x},\bm{\theta})$ is differentiable with respect to $\bm{\theta}$ at each fixed $\mathbf{x} \in \mathbb{R}^d$. 

\paragraph{Assumptions on the Learning Task}
To limit scope we focus on supervised learning tasks, where each datum is associated to a \emph{label} taking values in a set $\mathcal{Y}$.
Denote the training dataset $\{(\mathbf{x}_j,\mathbf{y}_j)\}_{j=1}^n \subset \mathbb{R}^d \times \mathcal{Y}$. 
Let $L : \mathcal{Y} \times \mathbb{R}^e \rightarrow \mathbb{R}$ be a loss function such that $L(\mathbf{y},f(\mathbf{x}))$ measures the loss incurred by using the model output $f(\mathbf{x})$ when the true label is $\mathbf{y} \in \mathcal{Y}$.
It will be assumed that $L$ is differentiable with respect to its second argument.

\paragraph{Probabilistic Notation}
Let $\mathcal{P}(\mathbb{R}^p)$ denote the set of (Borel) probability distributions on $\mathbb{R}^p$.
For $Q \in \mathcal{P}(\mathbb{R}^p)$, denote the (negative) entropy $\mathcal{E}(Q) = \int q(\bm{\theta}) \log q(\bm{\theta}) \; \mathrm{d}\bm{\theta}$ if $Q$ has density $q$ on $\mathbb{R}^p$, and $\infty$ otherwise.
Let $\delta_{\bm{\theta}} \in \mathcal{P}(\mathbb{R}^p)$ denote a point mass at $\bm{\theta} \in \mathbb{R}^p$, so that $Q_m = \frac{1}{m} \sum_{i=1}^m \delta_{\bm{\theta}_i}$ is the empirical distribution associated with the set $\{\bm{\theta}_i\}_{i=1}^m \subset \mathbb{R}^p$.
Let $T_\# Q$ denote the pushforward of the distribution $Q$ under the map $T$, i.e. $(T_\# Q)(S) = Q(T^{-1}(S))$ where $T^{-1}(S) = \{x : T(x) \in S\}$.

\paragraph{Gradient Notation}
For $F : \mathbb{R}^p \rightarrow \mathbb{R}$ and $\mathbf{F} : \mathbb{R}^p \rightarrow \mathbb{R}^p$, let $\nabla F$ denote the gradient of $F$ and let $\nabla \cdot \mathbf{F}$ denote the divergence of $\mathbf{F}$.
For $\mathcal{F} : \mathcal{P}(\mathbb{R}^p) \rightarrow \mathbb{R}$ and $Q \in \mathcal{P}(\mathbb{R}^p)$, the \emph{variational gradient} $\nabla_V \mathcal{F}(Q) : \mathbb{R}^p \rightarrow \mathbb{R}^p$ of $\mathcal{F}$ at $Q$, if it exists, is defined as the Euclidean gradient of the first variation $\mathcal{F}'(Q)$ of $\mathcal{F}$ at $Q$; i.e. $\nabla_V \mathcal{F}(Q)(\bm{\theta}) = \nabla_{\bm{\theta}} \mathcal{F}'(Q)(\bm{\theta})$ for each $\bm{\theta} \in \mathbb{R}^p$.

\subsection{Training as Distributional Optimisation}
\label{subsec: joint train}

Our starting point is the observation that the strategies for combining models from \Cref{sec: context} can be cast as minimisation of an appropriate loss function $\mathcal{L} : \mathcal{P}(\mathbb{R}^p) \rightarrow \mathbb{R}$ defined on the set of probability distributions $\mathcal{P}(\mathbb{R}^p)$.
To limit scope we assume the models $f_i$ are instances of the same architecture, differing only in their parameters, i.e. $f_i(\mathbf{x}) = f(\mathbf{x} , \bm{\theta}_i)$.
In the case of ensembles, the performance of the combined model on the training dataset can be captured by the distributional loss function
    \begin{align}
        \mathcal{L}(Q) = \sum_{j=1}^n L\left( \mathbf{y}_j , \int f(\mathbf{x}_j , \bm{\theta}) \; \mathrm{d}Q(\bm{\theta}) \right) .  \label{eq: L ensemble}
    \end{align}
    Indeed, taking $Q$ equal to $Q_m = \sum_{i=1}^m w_i \delta_{\bm{\theta}_i}$ recovers the ensemble model \eqref{eq: output ave}.
    Under mild assumptions (cf. \Cref{rem: well-posed}), direct optimisation of $\mathcal{L}$ over $\mathcal{P}(\mathbb{R}^p)$ is well-posed.
    Further, the minimiser will typically have more than one element in its support; this is because the ensemble prediction $\int f(\cdot , \bm{\theta}) \; \mathrm{d}Q(\bm{\theta})$ is a convex combination of models, and thus more expressive than any individual instance of the model.
    However, the support of the minimiser will typically be a finite set; see \citet[][e.g. Theorem 21 in Chapter 5]{lindsay1995mixture}.
    This discreteness renders direct optimisation of $\mathcal{L}$ extremely difficult.
As a second example, we can lift \ac{LoRA} averaging to a distributional optimisation task by identifying $\bm{\theta} = (\mathbf{A},\mathbf{B})$ and setting
    \begin{align} 
        \mathcal{L}(Q) \hspace{-3pt} = \hspace{-3pt} \sum_{j=1}^n L \left( \mathbf{y}_j ,  f\left( \mathbf{x}_j , \mathbf{W} \hspace{-3pt} + \hspace{-3pt} \int \mathbf{B} \mathbf{A} \; \mathrm{d}Q(\mathbf{A},\mathbf{B}) \right) \hspace{-3pt} \right)   \label{eq: loss LoRA}
    \end{align}
    so that $Q$ is a distribution on $\mathbb{R}^{r \times k} \times \mathbb{R}^{d \times r}$, and an $m$-particle discretisation $Q_m = \sum_{i=1}^m w_i \delta_{(\mathbf{A}_i,\mathbf{B}_i)}$ of $Q$ corresponds to using a rank-$mr$ adapter as in \eqref{eq: LoRA ave}.
    The case where $r = 1$ was considered in \citet{nitanda2025propagation}.
    Again, direct optimisation of $\mathcal{L}$ in this case is computationally intractable due to the discrete support of the minimising distribution $Q$.

To address the  difficulties with optimisation of $\mathcal{L}$, we can consider additional regularisation with (negative) entropy with $\lambda > 0$ as in \eqref{eq: dist opt}. 
The entropy term ensures that the minimiser is absolutely continuous with respect to Lebesgue measure; we empirically investigate the benefit of entropic regularisation in \Cref{sec: experiments}.

Although \eqref{eq: L ensemble} and \eqref{eq: loss LoRA} are structurally similar, they differ in a respect that matters for the numerical methods of \Cref{sec: algorithms}:

\begin{proposition*}[Convexity of $\mathcal{J}$]
\label{prop: convexity}
Suppose that, for each $j$, the map $\mathbf{u} \mapsto L(\mathbf{y}_j , \mathbf{u})$ is convex on $\mathbb{R}^e$.
Then $\mathcal{L}$ in \eqref{eq: L ensemble} is convex on $\mathcal{P}(\mathbb{R}^p)$, and for $\lambda > 0$ the objective $\mathcal{J}$ in \eqref{eq: dist opt} is strictly convex on $\mathcal{P}(\mathbb{R}^p)$.
\end{proposition*}
\begin{proof}
The map $Q \mapsto \int f(\mathbf{x}_j , \bm{\theta}) \, \mathrm{d}Q(\bm{\theta})$ is affine in $Q$, and the composition of a convex function with an affine map is convex; a finite sum of convex functions is convex, so $\mathcal{L}$ is convex on the convex set $\mathcal{P}(\mathbb{R}^p)$.
Since $\mathcal{E}$ is strictly convex, $\mathcal{J}$ is strictly convex for $\lambda > 0$.
\end{proof}

\noindent This proposition therefore holds in each of our experiments involving ensembles, but not for \ac{LoRA} averaging: although $Q \mapsto \int \mathbf{B}\mathbf{A} \, \mathrm{d}Q(\mathbf{A},\mathbf{B})$ is affine, the map $\mathbf{u} \mapsto L(\mathbf{y}_j , f(\mathbf{x}_j , \mathbf{W} + \mathbf{u}))$ is not convex in general, because $f$ depends non-linearly on its second argument.

\subsection{Particle-Based Methods}
\label{sec: algorithms}

Let $Q^\star$ denote a solution of \eqref{eq: dist opt}, which we assume to exist.
Since $Q^\star$ is implicitly defined as a minimiser of $\mathcal{J}$, numerical methods are needed.
To interpret the numerical solution as a combination of models, we require a numerical approximation to $Q^\star$ of the form $\sum_{i=1}^m w_i \delta_{\bm{\theta}_i}$; i.e. a \emph{particle-based} method.
(Note that we cannot plug a discrete distribution directly into \eqref{eq: dist opt}, as the entropy term will be infinite in general.)

The most well-studied particle-based method is \emph{mean-field Langevin dynamics} (\acs{MFLD}; cf. \Cref{sec: MFLD}), which can be interpreted as a coupled version of stochastic gradient descent. 
In this work we cast our horizons beyond \ac{MFLD} and also explore more recent advances in distributional optimisation, including \emph{variational gradient descent} (\acs{VGD}; cf. \Cref{sec: MFLD}), and proposing a novel algorithm in this setting called \emph{functional} \ac{VGD} (\Cref{subsec: fVGD}).

\begin{remark}[Existence and uniqueness of $Q^\star$]
\label{rem: well-posed}
    A standard way to ensure existence of a solution $Q^\star$ is to include a confining potential into the loss function $\mathcal{L}$; this is equivalent to regularisation using Kullback--Leibler divergence, cf. \Cref{app: KL ent}.
    If a solution $Q^\star$ exists, it will be unique whenever $\mathcal{J}$ is strictly convex.
\end{remark}

\subsubsection{Mean Field Langevin Dynamics}
\label{sec: MFLD}

\Acl{MFLD} refers to the \emph{McKean--Vlasov} process
\begin{align}
    \mathrm{d} \bm{\theta}_t = - \nabla_V  \mathcal{L}(Q_t)(\bm{\theta}_t) \mathrm{d} t + \sqrt{2 \lambda} \mathrm{d} \mathbf{B}_t ,  \label{eq: McK-V}
\end{align}
where $\nabla_V$ denotes the variational gradient (cf. \Cref{subsec: setup} and \Cref{app: var grad}), $Q_t = \mathrm{Law}(\bm{\theta}_t)$, and $\mathbf{B}_t$ is a Brownian motion on $\mathbb{R}^p$ \citep{hu2021mean}.
This process is constructed such that, under appropriate regularity conditions, the distribution $Q_t$ converges to $Q^\star$ in the large $t$ limit regardless of how $\bm{\theta}_0$ is initialised \citep[see e.g.][]{chizat2022mean}.
This dual representation of $Q^\star$ as the limiting distribution of a stochastic process opens the door to designing algorithms to approximate $Q^\star$; however, the dependence on $Q_t$ is problematic as this is not analytically available.
The McKean--Vlasov process can be approximated by plugging in an $m$ particle discretisation $Q_t \approx Q_m^t = \frac{1}{m} \sum_{j=1}^m \delta_{\bm{\theta}_j^t}$ and employing a suitable numerical method.
In particular, the Euler--Maruyama method with step size $\epsilon > 0$ leads to the following system, which can be interpreted as a coupled version of stochastic gradient descent:
\begin{align}
\bm{\theta}_i^{t+1} & = \bm{\theta}_i^t - \epsilon  \nabla_V \mathcal{L}(Q_m^t)(\bm{\theta}_i^t) + \sqrt{2 \lambda \epsilon} \mathbf{Z}_i^t   \label{eq: MFLD eps}
\end{align}
where the $\mathbf{Z}_i^t$ are independent standard Gaussian variables on $\mathbb{R}^p$.
Intuitively, we can think of \eqref{eq: MFLD eps} as approximating `stochastic gradient descent in the space of probability distributions'.
The distribution $Q_m^t$ converges (in a precise sense) to $Q^\star$ in the limit as $t \rightarrow \infty$ and $m \rightarrow \infty$; crucially, the available results require $\mathcal{L}$ to be convex on $\mathcal{P}(\mathbb{R}^p)$ \citep{chizat2022mean}, which by \Cref{prop: convexity} holds for ensembles, but does not hold for \ac{LoRA} averaging in general.
Implementational details are discussed in \Cref{app: implementation MFLD}.

\subsubsection{Variational Gradient Descent}
\label{sec: VGD}

The recent work of \citet{wang2019nonlinear} proposed a generalisation of Stein variational gradient descent \citep[SVGD;][]{liu2016stein} suitable for problems of the form \eqref{eq: dist opt}, which has the potential to be more efficient than \ac{MFLD}.
The idea is to perform \emph{deterministic} gradient descent on $\mathcal{J}$, noting that the directional derivative in the direction parametrised by the vector field $\mathbf{v} : \mathbb{R}^p \rightarrow \mathbb{R}^p$ is
\begin{align*}
    & \textstyle \left. \frac{\mathrm{d}}{\mathrm{d}\epsilon} \mathcal{J}((\mathbf{I} + \epsilon \mathbf{v})_\# Q) \right|_{\epsilon = 0} \\
    & = \int \bigl[ \nabla_V \mathcal{L}(Q)(\bm{\theta}) \cdot \mathbf{v}(\bm{\theta}) - \lambda (\nabla \cdot \mathbf{v})(\bm{\theta}) \bigr] \; \mathrm{d}Q(\bm{\theta}) .
\end{align*}
Let $k : \mathbb{R}^p \times \mathbb{R}^p \rightarrow \mathbb{R}$ be a symmetric positive definite kernel, such as the Gaussian kernel $k(\bm{\theta},\bm{\vartheta}) = \exp( - \ell^{-2} \|\bm{\theta} - \bm{\vartheta}\|^2 )$ with length-scale $\ell > 0$.
Let $\nabla_2 k$ denote the gradient with respect to the second argument of the kernel and let $\mathcal{H}_k$ denote the \ac{RKHS} associated to the kernel.
The \ac{VGD} algorithm selects, at each time $t$, the vector field $\bm{\phi} \in \mathcal{H}_k^p$ for which the descent is steepest, subject to $\|\bm{\phi}\|_{\mathcal{H}_k^d} = 1$. 
This leads to a deterministic interacting particle system, simulated in discrete time as $\bm{\theta}_i^{t+1} = \bm{\theta}_i^t + \epsilon \bm{\phi}(\bm{\theta}_i^t; Q_m^t)$, where
\begin{align}
\hspace{-5pt} \bm{\phi}(\bm{\theta}_i^t; Q_m^t) 
= \frac{1}{m} \sum_{j=1}^m \left[ \begin{array}{l} - k(\bm{\theta}_i^t , \bm{\theta}_j^t) \nabla_V \mathcal{L}(Q_m^t)(\bm{\theta}_j^t) \\ \quad + \lambda (\nabla_2 k)(\bm{\theta}_i^t,\bm{\theta}_j^t) \end{array} \right] \label{eq: vgd_ws}
\end{align}
Under conditions established in \citet{chazal2025computable}, the distribution $Q_m^t$ converges\footnote{The cited theory is formulated using Kullback--Leibler regularisation; this is equivalent to entropic regularisation when the loss includes the corresponding confining potential, as explained in \Cref{app: KL ent}.} in an averaged sense to $Q^\star$ in the limit $t \rightarrow \infty$ and $m \rightarrow \infty$.
One can consider \ac{VGD} as a \emph{de-randomisation} of \ac{MFLD}; taking the length-scale $\ell$ to zero, and solving the differential equation using the Euler method, we recover \eqref{eq: MFLD eps} with the Gaussian perturbation removed.
Implementational details are discussed in \Cref{app: implementation VGD}.

\subsubsection{Functional VGD}
\label{subsec: fVGD}

Although \ac{VGD} ameliorates the randomness of \ac{MFLD}, it has been argued that the use of a kernel on the parameter space introduces difficulties when the parameter $\bm{\theta}$ is high-dimensional \citep[see e.g.][in the special case of SVGD]{ba2021understanding}.
Motivated by \emph{repulsive} deep ensembles, which take a function-space perspective on the gradient flow of \eqref{eq: dist opt}, we also consider a \emph{functional} version of \ac{VGD} which operates on the output, rather than the parameters, of the machine learning model.
Previous work focussed on approximating a Bayesian posterior $\pi$ using an ensemble, for which the loss function in our notation is $\mathcal{L}(Q) = - \int q(\bm{\theta}) \log \pi(\bm{\theta}) \, \mathrm{d}\bm{\theta}$ and $\lambda = 1$ \citep{wang2018function, d2021repulsive}. 
Our contribution in this respect is first to develop the idea for general $\mathcal{L}$, decoupling it from the Bayesian framework, and then to empirically assess the performance of the functional approach (in \Cref{sec: experiments}).

Our starting point is to map each parameter
$\bm{\theta}\in\mathbb{R}^p$ to its corresponding
$\mathbb{R}^e$-valued model function. Let
$\mathfrak{F} = \{f(\cdot;\bm{\theta}): \bm{\theta}\in\mathbb{R}^p\}$ denote the resulting function space. The map $\Phi:\mathbb{R}^p\rightarrow\mathfrak{F}$ defined by $\Phi(\boldsymbol{\theta})=f(\cdot;\boldsymbol{\theta})$ allows us to associate each $Q\in\mathcal{P}(\mathbb{R}^p)$ with the pushforward distribution $\Phi_{\#}Q\in\mathcal{P}(\mathfrak{F})$.
We may similarly regard, with a slight abuse of notation, $\mathcal{L}:\mathcal{P}(\mathfrak{F})\rightarrow\mathbb{R}$, thereby decoupling the optimisation objective from the dimension $p$ of the parameter $\boldsymbol{\theta}$. 
Let $k:\mathfrak{F}\times\mathfrak{F}\rightarrow\mathbb{R}$ be a symmetric positive definite kernel. Analogous to~\eqref{eq: vgd_ws}, we can then write the kernelised steepest descent in function space as
\begin{align*}
    \bm{\phi}({f}_i^t; Q_m^t) 
    = \frac{1}{m} \sum_{j=1}^m \left[ \begin{array}{l} - k({f}_i^t , {f}_j^t) \nabla_V \mathcal{L}(Q_m^t)({f}_j^t) \\ \quad + \lambda (\nabla_2 k)({f}_i^t,{f}_j^t) \end{array} \right]
\end{align*}
where $f_j^t=f(\cdot;\boldsymbol{\theta}_j^t)\in\mathfrak{F}$ and $Q_m^t = \frac{1}{m} \sum_{j=1}^m \delta_{{f}_j^t}$. To interpret $\nabla_V\mathcal{L}$ and $\nabla_2k$, 
define the stacked output map $F_n(\boldsymbol{\theta})  =(f(\mathbf{x}_1;\boldsymbol{\theta})^\top,\ldots, f(\mathbf{x}_n;\boldsymbol{\theta})^\top)^\top \in\mathbb{R}^{ne}$.
Thus, on $\{\mathbf{x}_i\}_{i=1}^n$, each $f(\cdot;\boldsymbol{\theta})\in\mathfrak{F}$ is represented by $F_n(\boldsymbol{\theta})$, and the required gradients may be interpreted as Euclidean gradients on $\mathbb{R}^{ne}$. Writing $D_{\boldsymbol{\theta}}F_n(\boldsymbol{\theta}) \in\mathbb{R}^{ne\times p}$ for the corresponding Jacobian, its transpose is used to pull the function-space update back to the parameter space, i.e.,
\begin{align}
    \boldsymbol{\theta}_i^{t+1} = \boldsymbol{\theta}_i^t + \epsilon D_{\boldsymbol{\theta}}F_n(\boldsymbol{\theta}_i^t)^\top \bm{\phi}(f_i^t;Q_m^t).
    \label{eq: proj_fvgd}
\end{align}
The resulting algorithm will be called \emph{functional} VGD (\acs{FVGD}).
Note that \ac{FVGD} (\Cref{subsec: fVGD}) is applicable to ensembles but is \emph{not} applicable to \ac{LoRA} averaging, since in the latter case the averaging occurs \textit{before} the nonlinear transformation is applied.
Implementational details are discussed in \Cref{app: implementation FVGD}.

\smallskip

At this point we have formulated joint training as distributional optimisation and introduced several numerical methods for this task; to evaluate their performance we now undertake an empirical assessment.

\section{Experimental Results}
\label{sec: experiments}

Our experimental assessment begins with considering ensemble methods for simple classification tasks (\Cref{sec: ensembles}), before moving to \ac{LoRA} averaging for the more challenging task of fine-tuning foundational models (\Cref{sec: LoRA averaging}).

\subsection{Ensemble Methods for Classification}
\label{sec: ensembles}

For these first experiments we consider classification tasks and employ the cross-entropy loss $L(\bm{y},\bm{p}) = - \sum_i y_i \log(p_i)$ where $\bm{y}$ is a one-hot vector indicating the true label and $\bm{p}$ is a vector of predicted class probabilities; we take $f(\mathbf{x},\bm{\theta}) = \mathrm{logit}(\bm{p})$ to be the output from the machine learning model.
To mimic more challenging settings, we deliberately limit the capacity of the model architectures, so no individual model can perfectly solve the given task.
Two sets of results are presented; a two-dimensional \textbf{Spiral} classification task (\Cref{fig:spiral_all,fig:spiral_all_boundary}) and the \textbf{MNIST} dataset (\Cref{fig:mnist_all}).
All experimental protocol, including the architecture and training details, are contained in \Cref{app: ensemble protocol}.

\begin{figure*}[t!]
    \centering

    \begin{subfigure}[t]{0.32\linewidth}
        \centering
        \includegraphics[width=\linewidth]{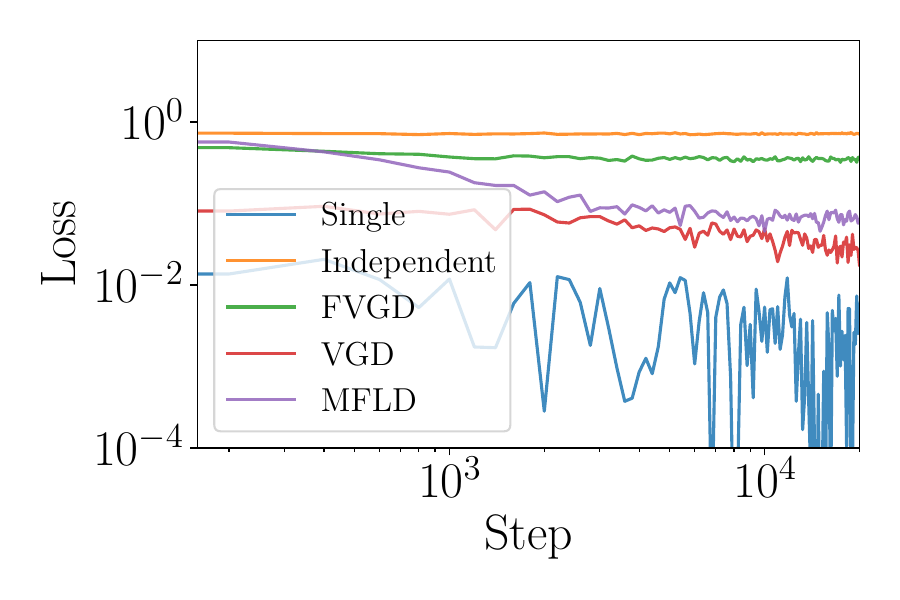}
        \caption{Training Loss}
        \label{fig:spiral_loss}
    \end{subfigure}
    \hfill
    \begin{subfigure}[t]{0.32\linewidth}
        \centering
        \includegraphics[width=\linewidth]{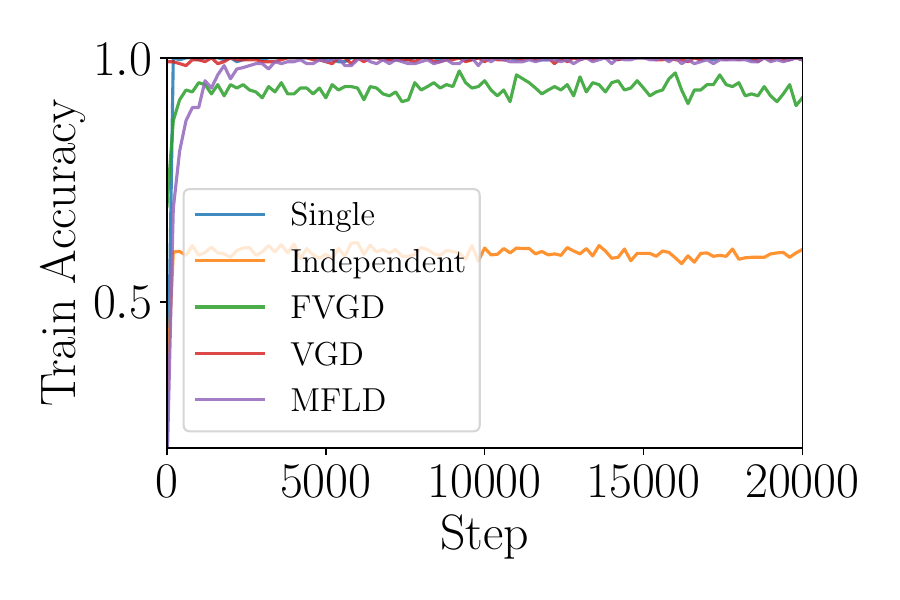}
        \caption{Training Accuracy}
        \label{fig:spiral_train_acc}
    \end{subfigure}
    \hfill
    \begin{subfigure}[t]{0.32\linewidth}
        \centering
        \includegraphics[width=\linewidth]{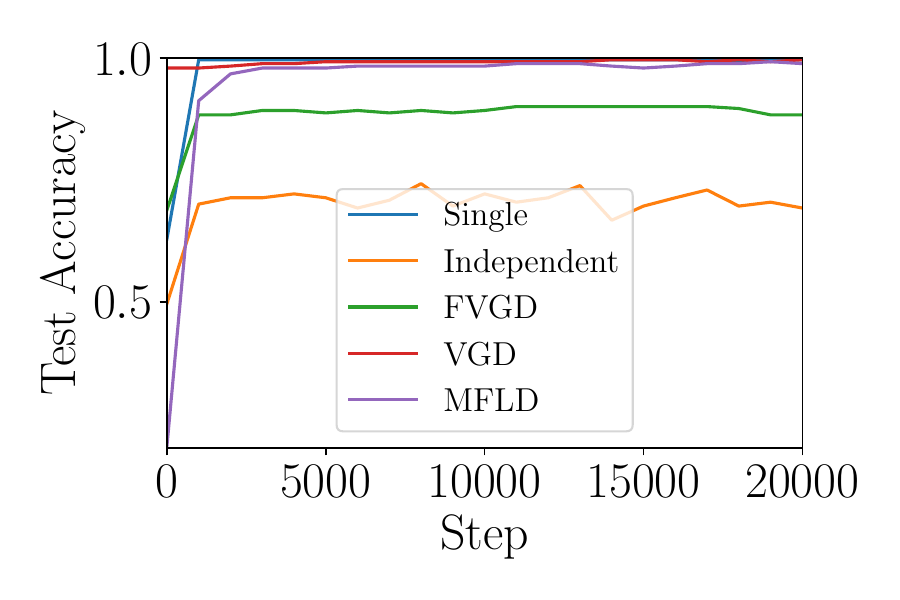}
        \caption{Testing Accuracy}
        \label{fig:spiral_test_acc}
    \end{subfigure}

    \caption{\textbf{Spiral} dataset: optimisation curves under a capacity-controlled setup. 
    Independent, MFLD, VGD, and FVGD all use an ensemble of $m=10$ MLPs with one hidden layer of width 2, while Single uses one hidden layer of width 20.
}
    \label{fig:spiral_all}
\end{figure*}

\begin{figure*}[t!]
    \centering

    \begin{subfigure}[t]{0.32\linewidth}
        \centering
        \includegraphics[width=\linewidth]{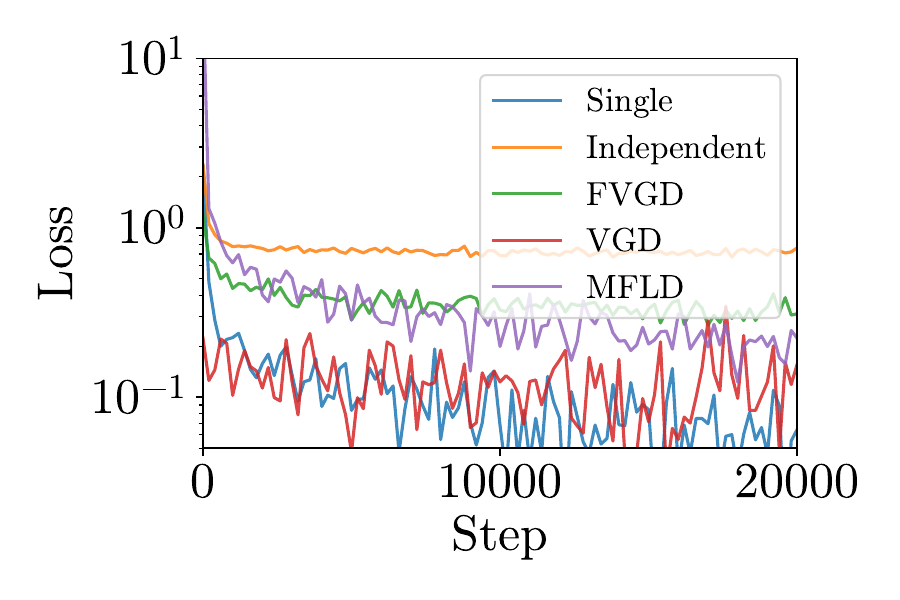}
        \caption{Training Loss}
        \label{fig:mnist_loss}
    \end{subfigure}
    \hfill
    \begin{subfigure}[t]{0.32\linewidth}
        \centering
        \includegraphics[width=\linewidth]{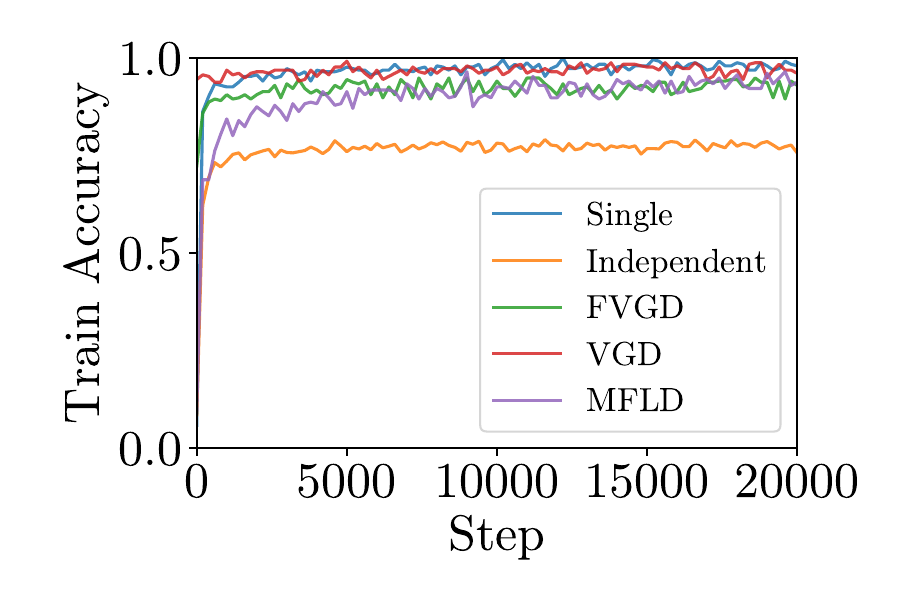}
        \caption{Training Accuracy}
        \label{fig:mnist_train_acc}
    \end{subfigure}
    \hfill
    \begin{subfigure}[t]{0.32\linewidth}
        \centering
        \includegraphics[width=\linewidth]{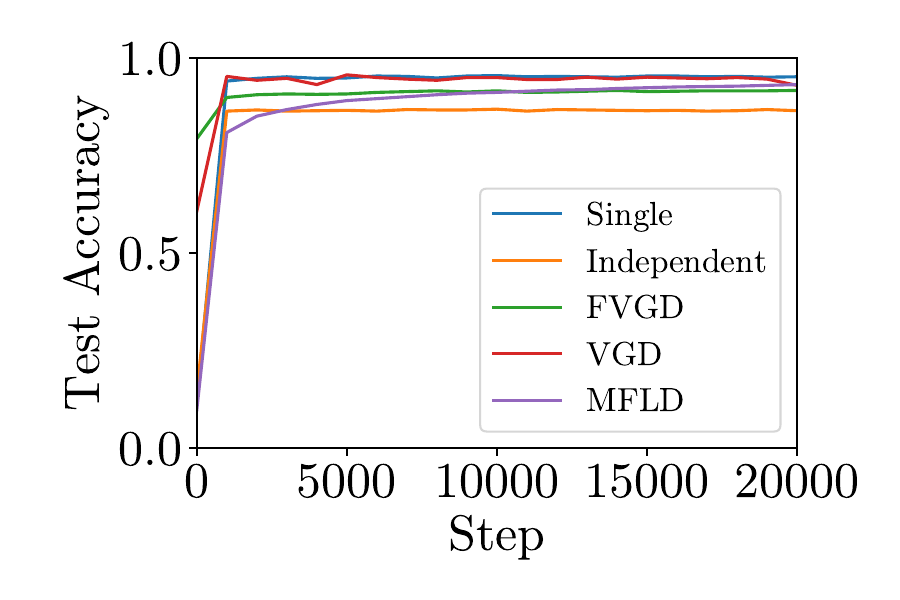}
        \caption{Testing Accuracy}
        \label{fig:mnist_test_acc}
    \end{subfigure}

    \caption{\textbf{MNIST} dataset: optimisation curves under a controlled capacity setup. 
    Independent, MFLD, VGD, and FVGD all use an ensemble of $m=10$ MLPs hidden layer of width 4, while Single is a larger MLP with hidden sizes [16,16].
}
    \label{fig:mnist_all}
\end{figure*}

\paragraph{Finding 1:  Joint training is beneficial}
First we confirmed that an ensemble of $m = 10$ independently trained models was inferior, on both \textbf{Spiral} and \textbf{MNIST}, to an ensemble in which the models are jointly trained (cf. \ac{MFLD}, \ac{VGD} or \ac{FVGD}): the independently trained ensemble attains the lowest test accuracy on both tasks (\Cref{fig:spiral_all,fig:mnist_all}), and the gap is substantial.
More notably, the jointly trained ensembles are competitive with a \emph{single} larger network trained end-to-end, despite decomposing into $m$ small models that admit distributed test-time evaluation: on both \textbf{Spiral} and \textbf{MNIST} the best joint ensembles essentially match the single model.
This is encouraging given that each constituent model is individually sub-optimal by design, and it motivates further investigation of joint training strategies for more challenging tasks.

\paragraph{Finding 2:  Performance is robust to the strength of entropic regularisation}

Entropic regularisation ($\lambda > 0$) is introduced primarily to render the distributional optimisation problem well-posed; 
the limiting case of \ac{MFLD} with $\lambda = 0$ corresponds to gradient descent on the joint training objective
\begin{align*}
        (\bm{\theta}_1 , \dots , \bm{\theta}_m) \mapsto \mathcal{L}(Q_m) = \sum_{j=1}^n L\left( \mathbf{y}_j , \frac{1}{m} \sum_{i=1}^m f(\mathbf{x}_j , \bm{\theta}_i)  \right)  
\end{align*}
and serves as a standard baseline.
In practice we implement stochastic mini-batching within \ac{MFLD} (cf. \Cref{app: implementation MFLD}), so that the $\lambda = 0$ baseline amounts to stochastic gradient descent.
A natural question is how sensitive downstream performance is to the choice of $\lambda$.
Sweeping $\lambda$ over several orders of magnitude on \textbf{MNIST} (\Cref{tab:mnist_lambda_ablation} in \Cref{app: vary lambda}), we find test accuracy to be remarkably stable for all three methods: \ac{VGD} remains at $0.948$ across the entire range $0 \le \lambda \le 20$, \ac{MFLD} stays within $[0.926, 0.932]$, and \ac{FVGD} within $[0.918, 0.923]$ for $\lambda \leq 10^{-3}$, with pronounced degradation appearing only for very large $\lambda$ (where \ac{FVGD} eventually collapses).
We therefore observe no accuracy penalty from the entropic term at moderate strengths, consistent with the view that parametrising the objective in terms of $Q$ (instead of $\bm{\theta}$), in which the entropic regulariser is \emph{convex}, yields a well-behaved optimisation landscape (\Cref{prop: convexity}).
On \textbf{MNIST}, \ac{VGD} also attains the highest accuracy of the three methods across this sweep.

\paragraph{Finding 3:  A functional perspective offers no consistent advantage}
Despite compelling theoretical arguments for taking a function-space perspective in related recent work (cf. \Cref{subsec: fVGD}), the performance of \ac{FVGD} is \emph{task-dependent} in our experiments.
On \textbf{Spiral}, \ac{FVGD} is inferior to \ac{MFLD} and \ac{VGD} (\Cref{fig:spiral_all}); on \textbf{MNIST}, however, it matches performance with \ac{MFLD} and is inferior to \ac{VGD} (test accuracy $0.918$-$0.923$, versus $0.932$ for \ac{MFLD} and $0.948$ for \ac{VGD} ; \Cref{tab:mnist_lambda_ablation}).
We therefore find no consistent benefit to operating in the functional output space of the model rather than in its original parameter space (i.e. $Q \in \mathcal{P}(\mathbb{R}^p)$) in this setting.
In \Cref{app: fvgd diagnosis} we diagnose that the Jacobian pull-back $(\nabla_2 f) \bm{\phi}$ is close to orthogonal to the intended functional direction for the narrow networks considered in these experiments; we regard the identification of a pull-back that avoid this failure modes as a natural step for future work.

\subsection{LoRA Averaging for Fine-Tuning LLMs}
\label{sec: LoRA averaging}

As a more challenging example, we now consider \ac{LoRA} averaging for fine-tuning \acp{LLM}, so that $\bm{\theta}_i = (\mathbf{A}_i,\mathbf{B}_i)$ where\footnote{For our experiments we fine-tune many layers, not just one, but we avoid making this explicit in the presentation.} each $\mathbf{A}_i$ and $\mathbf{B}_i$ has rank $r$.
Throughout this section, $m$ denotes the number of particles in a single distributional optimisation, $r$ the (maximum) rank of each particle, $R = mr$ the resulting (maximum) rank of the \ac{LoRA} averaged adapter, and $M$ the number of \emph{independently} trained adapters that are subsequently merged (\Cref{sec: LoRA averaging}, Finding 2).
For training and evaluation we used the commonsense reasoning suite assembled by \citet{hu2023llmadapters}, covering binary question answering \citep{clark2019boolq}, physical \citep{bisk2020piqa} and social \citep{sap2019socialiqa} commonsense, contextual completion \citep{zellers2019hellaswag}, pronoun disambiguation \citep{sakaguchi2021winogrande}, and science-style multiple-choice QA \citep{clark2018arc,mihaylov2018openbookqa}, while enabling a unified discriminative evaluation protocol.
Our experimental protocol follows \citet{nitanda2025propagation}, taking the training loss $L$ to be the standard per-token cross-entropy over the answer set, and we adopt Llama-3-8B \citep{meta2024llama3-8b,DBLP:journals/corr/abs-2407-21783} as our primary base model.
The architecture and training details are contained in \Cref{app: LLM protocol}.
An ablation study using alternative base models can be found in \Cref{app: choice of LLM}.

\paragraph{Finding 1: \ac{VGD} outperforms \ac{MFLD} at higher rank}

First we implemented \ac{LoRA} averaging with $m = 32$ rank-$1$ particles, so that the \ac{LoRA} averaged adapter $\Delta \mathbf{W} = \frac{1}{m} \sum_{i=1}^m \mathbf{B}_i \mathbf{A}_i$ has rank at most $R = mr = 32$, performing optimisation using either \ac{MFLD} or \ac{VGD} and comparing against Standard fine-tuning of a single rank-$32$ adapter with noisy AdamW.
The use of \ac{MFLD} in this context coincides with the work of \citet{nitanda2025propagation}, while the use of \ac{VGD} is novel, with a per-step cost comparable to \ac{MFLD} (a ${\sim}6\%$ wall-clock overhead; \Cref{app: timing}).
Regularisation of $\lambda = 10^{-5}$ was used for \ac{MFLD}, following \citet{nitanda2025propagation}, while for \ac{VGD} we used $\lambda = 10^{-7}$ at rank $32$ and $\lambda = 10^{-6}$ at rank $256$, following our ablation study in \Cref{app: vary lambda}.

At rank $R=32$ the three approaches are close in performance; the mean $8$-task accuracies are $84.97 \pm 0.09$ (Standard), $85.16 \pm 0.08$ (\ac{MFLD}) and $85.48 \pm 0.06$ (\ac{VGD}), cf block 1 in \Cref{tab:llm-main-best}.
However, the \emph{shape} of the performance distribution differs across methods (\Cref{fig:boxplot_seeds} in \Cref{app: visualise boxplot}); the medians are $85.00\%$ (Standard), $85.15\%$ (\ac{MFLD}) and $85.49\%$ (\ac{VGD}).
The best single-seed average performance was achieved by \ac{VGD} ($85.66\%$; \Cref{tab:per_seed_full}).
A paired across-task bootstrap (\Cref{tab:llm-paired-bootstrap} in \Cref{app: visualise boxplot}) places \ac{VGD} above Standard ($+0.52$, confidence interval $[+0.31,+0.74]$) and \ac{MFLD} ($+0.32$, $[+0.16,+0.48]$), with both comparisons statistically significant.

At higher rank there is a clearer advantage for \ac{VGD}.
For a single rank-$256$ adapter (block 2 in \Cref{tab:llm-main-best}), \ac{VGD} reaches $83.54\%$ against $81.83\%$ for \ac{MFLD} and $82.75\%$ for standard fine-tuning, and the paired across-task comparison (\Cref{tab:llm-paired-bootstrap} in \Cref{app: visualise boxplot}) places \ac{VGD} above both \ac{MFLD} ($+1.71$, $[+1.38,+2.04]$) and Standard ($+0.79$, $[+0.38,+1.24]$), whereas \ac{MFLD} falls below Standard at this rank ($-0.92$, $[-1.20,-0.63]$).
The $\lambda$-sweep in \Cref{tab:full_lambda_sweep_appendix} of \Cref{app: vary lambda} shows that some of this gain is already present at $\lambda = 0$ ($82.92$, versus $82.75$ for Standard), while the best result is $83.54$ at the selected $\lambda = 10^{-6}$.
Setting $\lambda = 0$ in \eqref{eq: vgd_ws} does \emph{not} recover Standard fine-tuning; it leaves a kernel-smoothed gradient, which coincides with the ordinary gradient only in the limit $\ell \rightarrow 0$.
Since the $\lambda = 0$ result ($82.92$) already exceeds Standard fine-tuning ($82.75$), this is consistent with a contribution from the kernel-induced coupling between particles; the positive-$\lambda$ optimum ($83.54$) suggests that entropic regularisation is beneficial.

\begin{SCtable*}[\sidecaptionrelwidth][h]
\caption{LoRA fine-tuning on Llama-3-8B (accuracy, \%).
Best per column per block in bold.
The final column $\Delta_{\text{Std}}$ is the mean per-task accuracy gain over the Standard baseline within each block; bootstrap confidence intervals for these deltas are given in \Cref{tab:llm-paired-bootstrap}.
}
\label{tab:llm-main-best}
\centering
\scriptsize
\setlength{\tabcolsep}{3pt}
\begin{tabular}{lcccccccccc}
\toprule
Method & SIQA & PIQA & Wino & OBQA & ARC-c & ARC-e & BoolQ & Hella & Avg & $\Delta_{\text{Std}}$ \\
\midrule

\multicolumn{11}{l}{\emph{Single rank-32 adapter (mean over 8 seeds)}} \\
Standard
& 80.23 & 88.38 & 86.28 & 85.67 & 78.83 & 90.26 & 74.82 & 95.26 & 84.97 & --- \\
MFLD
& 80.29 & 88.28 & 86.39 & 86.09 & \textbf{79.83} & 90.26 & 74.95 & 95.22 & 85.16 & $+0.20$ \\
VGD
& \textbf{80.86} & \textbf{88.86} & \textbf{86.48} & \textbf{86.64}
& \textbf{79.83} & \textbf{90.75} & \textbf{75.00} & \textbf{95.44}
& \textbf{85.48} & $\mathbf{+0.52}$ \\

\midrule
\multicolumn{11}{l}{\emph{Single rank-256 adapter}} \\
Standard
& 80.23 & 86.99 & 84.39 & 81.54 & 74.85 & 87.94 & \textbf{73.40} & 92.69 & 82.75 & --- \\
MFLD
& 78.79 & 86.06 & 84.24 & 80.98 & 73.38 & 86.94 & 72.36 & 91.89 & 81.83 & $-0.92$ \\
VGD
& \textbf{80.60} & \textbf{87.69} & \textbf{85.87} & \textbf{83.50}
& \textbf{75.65} & \textbf{88.51} & 73.18 & \textbf{93.34}
& \textbf{83.54} & $\mathbf{+0.79}$ \\

\midrule
\multicolumn{11}{l}{\emph{PoC merge of 8 rank-32 adapters $\to$ rank-256}} \\
Standard
& 82.12 & 89.74 & 88.48 & 88.24 & 82.91 & 92.16 & 76.05 & 96.27 & 87.00 & --- \\
MFLD
& 82.18 & \textbf{90.14} & 88.73 & 88.08 & \textbf{83.35}
& 92.17 & 76.25 & 96.22 & 87.14 & $+0.14$ \\
VGD
& \textbf{82.55} & 90.10 & \textbf{88.97} & \textbf{88.49}
& 83.18 & \textbf{92.36} & \textbf{76.39} & \textbf{96.41}
& \textbf{87.31} & $\mathbf{+0.31}$ \\

\midrule
\multicolumn{11}{l}{\emph{Joint OA (8 branches jointly trained, $8{\times}$ inference)}} \\
MFLD
& 82.29 & \textbf{90.53} & 90.77 & 89.11 & 84.47
& \textbf{93.27} & \textbf{85.48} & 97.18 & 89.14 & --- \\
VGD
& \textbf{82.87} & 90.51 & \textbf{91.07} & \textbf{89.47}
& \textbf{84.69} & 93.21 & 85.25 & \textbf{97.59}
& \textbf{89.33} & --- \\

\bottomrule
\end{tabular}
\end{SCtable*}

\paragraph{Finding 2: PoC merging is beneficial}

Recent work revealed that merging $M$ adapters $\Delta \mathbf{W}_i$, each of rank at most $R$, into a single adapter $\frac{1}{M} \sum_{i=1}^M \Delta \mathbf{W}_i$ of rank at most $MR$, can improve performance beyond training a standard rank-$MR$ adapter; this was called a \emph{propagation of chaos} (\acs{PoC}) \emph{merge} in \citet{nitanda2025propagation}, where each adapter $\Delta \mathbf{W}_i$ was independently trained using \ac{MFLD}.
We test whether this benefit carries over to \ac{VGD}.
Here we train $M=8$ independent adapters $\Delta \mathbf{W}_i$, each of rank at most $R = 32$, as previously described and \ac{PoC} merge them into a single adapter of rank at most $RM = 256$.
Note that the merged adapter has the same test-time cost as a standard rank-256 \ac{LoRA}.
Results are reported in block 3 of \Cref{tab:llm-main-best}.

All three training regimes benefit substantially from \ac{PoC} merging, with average accuracy gains of $+2.03$, $+1.98$, and $+1.82$ percentage points for Standard, \ac{MFLD}, and \ac{VGD} respectively over their corresponding single-adapter baselines.
This is consistent with the finding of \citet{nitanda2025propagation} that training $M$ smaller adapters and merging them provides an effective alternative to training a single adapter of comparable rank.
After merging, the three methods are close ($87.00\%$ for Standard, $87.14\%$ for \ac{MFLD} and $87.31\%$ for \ac{VGD}); the paired across-task comparison places \ac{VGD} above both Standard ($+0.31$, $[+0.24,+0.39]$) and \ac{MFLD} ($+0.17$, $[+0.04,+0.28]$), with \ac{MFLD} above Standard ($+0.14$, $[+0.01,+0.28]$).
In summary, \ac{PoC} merging appears broadly beneficial.
The effect of kernel choice for \ac{VGD} is explored in \Cref{app: kernel ablation}.

\paragraph{Finding 3:  Joint training yields further improvement}

The main drawback of \ac{PoC} merging is that the adapters $\Delta \mathbf{W}_i$ are independently trained, which could be sub-optimal.
In principle our methods enable joint training, so here we investigate the performance of a \ac{LoRA} \emph{ensemble} when jointly trained\footnote{The efficacy of ensembling \emph{independently} \ac{LoRA}-fine-tuned models was demonstrated in \citet{wang2023ensemble} but to our knowledge joint training has not previously been considered in this context.}.
That is, we now consider the setting where $\bm{\theta}_i = \mathbf{B}_i\mathbf{A}_i$ has rank at most $r = 32$ and we jointly train an ensemble of $m = 8$ such adapters using either \ac{MFLD} or \ac{VGD}, similarly to \Cref{sec: ensembles}.
Results are reported in block 4 of \Cref{tab:llm-main-best}.
At an $8\times$ training budget, the jointly trained ensemble (logit averaging across all $m = 8$ branches) achieves $89.14\%$ (\ac{MFLD}) and $89.33\%$ (\ac{VGD}) average accuracy on the commonsense reasoning benchmark, exceeding the corresponding $1\times$-training \ac{PoC} merges by $+2.00$ and $+2.03$ percentage points respectively.
This improvement is obtained at the higher $8\times$ training cost. 

\section{Discussion}
\label{sec: discussion}

Despite it being common practice, the potential benefit of combining predictions from different models remains poorly understood, with contrasting arguments for the effectiveness of this strategy being put forward \citep[e.g.][]{abe2022deep,wortsman2022model}.
To gain a deeper understanding we need to consider how the collection of models are trained \citep{mattei2025ensembles}, and in this work we have presented distributional optimisation as a framework in which joint training can be conceptualised.
This perspective brings new insight:
First, it reveals that ensembles and \ac{LoRA} averaging, which are usually discussed interchangeably, are not equivalent once lifted to $\mathcal{P}(\mathbb{R}^p)$; the former yields a convex objective for which the convergence theory of \ac{MFLD} applies, the latter does not.
Second, this perspective suggests routes to developing new algorithms, exemplified by \ac{VGD} and \ac{FVGD}.
The potential of these methods was empirically explored, finding that \ac{VGD} often matched or out-performed \ac{MFLD}.

Our results also weigh in on the argument above, demonstrating that jointly trained ensembles can be competitive with a single larger model while admitting distributed test-time evaluation (Finding 1 in \Cref{sec: ensembles}), and that the resulting distributional objective is well-behaved across a wide range of entropic-regularisation strengths (Finding 2 in \Cref{sec: ensembles}).

\paragraph{Limitations and Opportunities}

The main limitation of our work is that the empirical assessment was proof-of-concept level; further empirical investigation on more challenging learning tasks will be required to comprehensively assess the benefit of the distributional optimisation framework.
We did not emphasise geometric considerations, but it can be shown that the \ac{VGD} method exploits \emph{Stein} geometry \citep{wang2019nonlinear}, while recent work suggests that optimisation using the \emph{Wasserstein} geometry may offer superior convergence properties in high-dimensional settings \citep{he2024regularized,duncan2023geometry,d2021repulsive}.
Investigating alternatives to \ac{VGD} which exploit the Wasserstein geometry, as well as extending \ac{FVGD} to \ac{LLM} fine-tuning, are promising directions for future work.

%% file: acks.tex
YL was supported by a five-year PGTA scholarship from the School of Computing, Newcastle University.
CW was supported by the China Scholarship Council under Grant Number 202208890004. 
CJO, ZS were supported by EPSRC EP/W019590/1. 
CJO was supported by a Philip Leverhulme Prize PLP-2023-004.

%% file: appendix.tex
\section*{Supplementary Material}

This document contains supplementary material for the manuscript \textit{A Distributional Optimisation Perspective on Combining Models in Deep Learning}.
\Cref{app: other approaches} contains a discussion of other approaches to combining models in deep learning.
\Cref{app: KL ent} explains how entropic regularisation and Kullback--Leibler regularisation are equivalent up to a re-defining of the loss function $\mathcal{L}$.
\Cref{app: implementation} contains full details on how each algorithm was implemented.
The protocol for each of our experiments is reported in \Cref{app: protocol}, with additional empirical results contained in \Cref{app: additional}.

\section{Combining Models in Deep Learning}
\label{app: other approaches}

This appendix provides a broader discussion of alternative methods for combining multiple deep learning models, complementing the more focused discussion in the main text.

\paragraph{Mixture of Experts}

A \emph{mixture of experts} (\acs{MoE}) takes the form
\begin{align}
    f(\mathbf{x}) = \sum_{i=1}^m w_i(\mathbf{x}) f_i(\mathbf{x})   \label{eq: MoE}
\end{align}
where, in contrast to ensembles \eqref{eq: output ave}, the weights $w_i$ are now input-dependent.
As with ensembles, weights $w_i(\mathbf{x})$ can be learned based on a pre-trained candidate model set, but it is also common for both the models $f_i$ and the weight functions $w_i$ to be jointly trained \citep{eigen2013learning,chen2022towards}.
\acp{MoE} are widely used in \acp{LLM}, where each model $f_i$ is trained on text data concerning a specific subject \citep{fedus2022review}.
On the other hand, one can interpret the \ac{MoE} in \eqref{eq: MoE} as a particular architecture choice for a single model, with the number $m$ of experts being a hyperparameter of the model.
Relatedly, it has been argued that a trained multilayer perceptron approximates a \ac{MoE} \citep{boix2025secret}.
One of the main challenges in \acp{MoE} is the use of \emph{ad hoc} methods to ensure coverage of the input domain; i.e. to ensure there is at least one expert capable of responding appropriately to a given input.

\paragraph{Parameter Averaging}

Assuming the models $f_i$ are instances of the same architecture, differing only in their parameters, i.e. $f_i(\mathbf{x}) = f(\mathbf{x} , \bm{\theta}_i)$, and that the parameter space is a vector space such (e.g. $\mathbb{R}^p$), one can construct a model 
\begin{align}
\mathbf{x} \mapsto f\left( \mathbf{x} , \sum_{i=1}^m w_i \bm{\theta}_i \right)   \label{eq: param aver}
\end{align}
based on a (possibly weighted) average over the parameters of each agent.
Note that, in contrast to ensembles and \acp{MoE}, parameter averaging does not change the expressive capacity of the model.
A special case of this approach is where $(\bm{\theta}_i)_{i=1}^m$ represents a single training trajectory; remarkably, maintaining even a uniform running average of parameters during training can improve generalisation \citep[e.g.][]{szegedy2016rethinking,izmailov2018averaging}.
Another common case is where each $\bm{\theta}_i$ arises from training using a different random seed \citep{nagarajan2019uniform,matena2022merging,neyshabur2020being,von2021neural,frankle2020linear}.
A uniform average will be sub-optimal in general, and the term \emph{model soup} refers to strategies used to learn non-uniform weights, which have been shown to improve performance on tasks such as image and text classification \citep{wortsman2022model}.
Again, a limiting instance of this approach is to use what is believed to be the single best-performing agent.
The main limitation of parameter averaging is that the loss landscape need not be convex in a neighbourhood of the candidate parameter set $\{\bm{\theta}_i\}_{i=1}^m$, in which case taking a convex combination of parameters could have an undesirable effect \citep{lu2026model}.

\paragraph{Combining Models by Voting}

In the setting of classification, \emph{majority vote} aggregates the outputs of several classifiers and selects the class that receives the most votes.
This can be considered a robust alternative to directly averaging the class specific probabilities produced by each model, which in turn we can represent as an ensemble \eqref{eq: output ave}.
Indeed, allowing each model a single `vote' limits the influence that any single model can have on the overall output.
The motivation for this approach is sometimes referred to as the \emph{wisdom of crowds}; however, this effect does not always exist when the models in the `crowd' are correlated and do not perform well \citep{orzechowski2025crowd}.
This issue can be addressed by allowing for non-uniform voting influence, where models which are believed to be better are afforded more influence in the voting process; this naturally engenders a trade-off between efficiency and robustness \citep{dogan2019weighted}.

\paragraph{Model Selection}

A limiting case of combining multiple models is to select the single best-performing model from the candidate set.
A richer candidate model set in principle offers an improved chance to find a model which performs well for the task at hand, but the selecting of an appropriate model can become more difficult when data are limited.
The idea is popular in statistical epidemiology and causal inference, where using cross-validation to select among candidate regression models is often called a \textit{Super Learner} \citep{van2007super}.
The approach is seen as alleviating the burden on the researcher to commit to a single regression model for their analysis, explaining its popular appeal.

\paragraph{Mixed Strategies}

Several works combine elements from the different approaches to combining machine learning models which we have discussed.
It is impractical to present an exhaustive discussion, but we highlight two relevant examples:
As a first example, an \ac{MoE} approach can be combined with \ac{LoRA}, where each expert is a fine-tuned foundational model with fine-tuning achieved using \ac{LoRA} and a \emph{gating network} used to delineate which `expert' is used at test-time \citep{wu2024mixture}.
As a second example, a mixed strategy combining elements of ensembles, \acp{MoE}, and \ac{LoRA} adapters was proposed in \citep{wang2023ensemble}.

\section{From Kullback--Leibler Divergence to Entropy}
\label{app: KL ent}

Several works, including \citet{chazal2025computable}, employ Kullback--Leibler divergence as a regulariser, as an alternative to the (negative) entropy used in this work. 
The purpose of this appendix is to clarify that one can interchange between these two different perspectives, interpreting the use of Kullback--Leibler divergence as the addition of an additional (linear) term in the loss function:

\begin{proposition*}[Confining potentials and relative entropy]
\label{prop: kl-confinement}
Let $V:\mathbb{R}^p\to\mathbb{R}$ be measurable and suppose that
$$
Z_V :=  \int \exp\left\{-\frac{V(\bm{\theta})}{\lambda}\right\} \,\mathrm{d}\bm{\theta} <\infty.
$$
Define the reference probability measure $Q_0$ via the density function
$$
q_0(\bm{\theta}) = \frac{1}{Z_V}  \exp\left\{-\frac{V(\bm{\theta})}{\lambda}\right\}.
$$
Then, for every probability measure $Q$ that is absolutely continuous with respect to Lebesgue measure, 
$$
\mathcal{L}(Q) + \lambda \operatorname{KLD}(Q \| Q_0) = \underbrace{ \mathcal{L}(Q) + \int V(\bm{\theta})\,\mathrm{d}Q(\bm{\theta}) + \lambda\log Z_V }_{ =: \tilde{\mathcal{L}}(Q) } + \lambda \mathcal{E}(Q)  
$$
whenever these quantities are well-defined.
\end{proposition*}
\begin{proof}
Let $Q$ have density $q$ on $\mathbb{R}^p$.
Since
$$
\log q_0(\bm{\theta}) = -\frac{V(\bm{\theta})}{\lambda} - \log Z_V,
$$
we have
\begin{align*}
\lambda\operatorname{KLD}(Q\Vert Q_0) &= \lambda \int q(\bm{\theta}) \log\frac{q(\bm{\theta})}{q_0(\bm{\theta})} \,\mathrm{d}\bm{\theta} \\
&= \lambda\mathcal{E}(Q) + \int V(\bm{\theta})\,\mathrm{d}Q(\bm{\theta}) + \lambda\log Z_V.
\end{align*}
Rearranging gives the result.
\end{proof}

\noindent Thus we can view the use of Kullback--Leibler regularisation as a modification to the loss function $\mathcal{L}$ through the introduction of a \emph{confining potential} $\int V(\bm{\theta}) \, \mathrm{d}\bm{\theta}$.
Note that the constant $\lambda\log Z_V$ does not affect the minimiser of the objective and can therefore be discarded for optimisation purposes.

\section{Implementational Detail}
\label{app: implementation}

This appendix contains details for how \ac{MFLD} (\Cref{app: implementation MFLD}), \ac{VGD} (\Cref{app: implementation VGD}), and \ac{FVGD} (\Cref{app: implementation FVGD}) were implemented.
Since each method requires access to the variational gradient of the distributional loss function $\mathcal{L}$, we first explain how this is calculated in \Cref{app: var grad}.

\subsection{Computing the Variational Gradient}
\label{app: var grad}

Here we explain how the variational gradients of distributional loss functions can be computed.
Direct from the definition of variational gradient in \Cref{subsec: setup}, together with the chain rule, the variational gradient of the ensemble loss function \eqref{eq: L ensemble} is
\begin{align}
        \nabla_V \mathcal{L}(Q)(\bm{\theta}) = \sum_{j=1}^n (\nabla_2 L)\left( \mathbf{y}_j , \int f(\mathbf{x}_j , \bm{\vartheta}) \; \mathrm{d}Q(\bm{\vartheta}) \right) (\nabla_2 f)(\mathbf{x}_j , \bm{\theta}) , \label{eq: grad ensemble loss}
\end{align}
while the variational gradient of the \ac{LoRA} averaging loss function \eqref{eq: loss LoRA} is
\begin{align*}
        \nabla_V \mathcal{L}(Q)(\bm{\theta}) & = \sum_{j=1}^n (\nabla_2 L) \left( \mathbf{y}_j ,  f\left( \mathbf{x}_j , \mathbf{W} + \int \tilde{\mathbf{B}} \tilde{\mathbf{A}} \; \mathrm{d}Q(\tilde{\mathbf{A}},\tilde{\mathbf{B}}) \right) \right) \\
        & \hspace{30pt} \times (\nabla_2 f)\left( \mathbf{x}_j , \mathbf{W} + \int \tilde{\mathbf{B}} \tilde{\mathbf{A}} \; \mathrm{d}Q(\tilde{\mathbf{A}},\tilde{\mathbf{B}}) \right) \left( \frac{\partial}{\partial \mathbf{A}} , \frac{\partial}{\partial \mathbf{B}} \right) (\mathbf{B} \mathbf{A})
\end{align*}
where $\bm{\theta} = (\mathbf{A},\mathbf{B})$.
The individual terms appearing in these gradients can be efficiently computed using Jacobian-vector products in parallel.
However, since we will only query the variational gradient on discretely supported measures $Q_m$, it is also possible to calculate variational gradients in a single step using the fact that
\begin{align}
\nabla_V \mathcal{L}(Q_m)(\bm{\theta}_i) = m \nabla_{\bm{\theta}_i} \mathcal{L}(Q_m), \qquad Q_m = \frac{1}{m} \sum_{i=1}^m \delta_{\bm{\theta}_i} . \label{eq: simple var grad}
\end{align}
This simpler strategy incurs a higher memory cost due to the need to work with the augmented parameter $\bm{\theta}_{1:m}$ vector of length $md$, but is easier to implement.
For our experiments in \Cref{sec: experiments}, the simple strategy \eqref{eq: simple var grad} was used for both \ac{MFLD} and \ac{VGD}.
Indeed, letting
\begin{align}
    \bm{\theta}_{1:m} = \left[ \begin{array}{c} \bm{\theta}_1 \\ \vdots \\ \bm{\theta}_m \end{array} \right] , \; F(\bm{\theta}_{1:m}) = - m \mathcal{L}(Q_m)  , \; (\nabla F)(\bm{\theta}_{1:m}) = \left[ \begin{array}{c} \nabla_{\bm{\theta}_1}F(\bm{\theta}_{1:m}) \\ \vdots \\ \nabla_{\bm{\theta}_m}F(\bm{\theta}_{1:m}) \end{array} \right] ,   \label{eq: def F}
\end{align}
the update equation \eqref{eq: MFLD eps} for \ac{MFLD} becomes
\begin{align}
    \bm{\theta}_{1:m}^{t+1} = \bm{\theta}_{1:m}^t + \epsilon (\nabla F)(\bm{\theta}_{1:m}^t) + \sqrt{2 \lambda \epsilon} \mathbf{Z}_{1:m}^t  \label{eq: vector MFLD}
\end{align}
where the $\mathbf{Z}_{1:m}^t$ are standard Gaussian.
On the other hand, for the experiments involving \ac{FVGD} in \Cref{sec: ensembles}, the memory-efficient implementation of \eqref{eq: grad ensemble loss} was used due to the additional requirement to compute Jacobian-vector products involving $\nabla_2 f$ in \eqref{eq: proj_fvgd}.

\subsection{Mean Field Langevin Dynamics}
\label{app: implementation MFLD}

Following standard practice in deep learning, we employed stochastic gradients based on a minibatch of size $B = 256$ (for the experiments in \Cref{sec: ensembles}) or $B = 16$ (for the experiments in \Cref{sec: LoRA averaging}).
In addition, for the experiments in \Cref{sec: LoRA averaging} we in practice implement a momentum-based optimiser \citep[Adam;][]{kingma2014adam} in preference to the stochastic gradient descent in \eqref{eq: vector MFLD}, following several other authors including \citet{nitanda2025propagation}.

For the experiments in \Cref{sec: ensembles}, \ac{MFLD} was run for $20{,}000$ steps with step size $\epsilon = 0.1$ on both \textbf{Spiral} and \textbf{MNIST}.
For the experiments in \Cref{sec: LoRA averaging}, the full set of hyperparameters is given in \Cref{tab:setup_hyperparams}; in summary, $3$ epochs at learning rate $\epsilon = 10^{-4}$ with $\lambda = 10^{-5}$, following \citet{nitanda2025propagation}.

\subsection{Variational Gradient Descent}
\label{app: implementation VGD}

For a basic implementation of \ac{VGD} using automatic differentiation, let $F$ be defined as in \eqref{eq: def F}, and in addition let $[\mathbf{K}(\bm{\theta}_{1:m})]_{i,j} = k(\bm{\theta}_i , \bm{\theta}_j)$ and $[\nabla_2 \mathbf{K}(\bm{\theta}_{1:m})]_{i,j} = (\nabla_2 k)(\bm{\theta}_i , \bm{\theta}_j)$.
Then the system of \acp{ODE} can be written as
\begin{align}
    \frac{ \mathrm{d}\bm{\theta}_{1:m} }{ \mathrm{d}t } & = \frac{1}{m} \mathbf{K}(\bm{\theta}_{1:m}) (\nabla F)(\bm{\theta}_{1:m}) + \frac{\lambda}{m} (\nabla_2 \mathbf{K})(\bm{\theta}_{1:m}) \mathbf{1} \label{eq: vgd ode}
\end{align}
and any suitable numerical method for solving \acp{ODE} can be applied.
In practice we implement a stochastic gradient momentum-based optimiser \citep[Adam;][]{kingma2014adam} in preference to an Euler discretisation of \eqref{eq: vgd ode}, following standard practice for Stein variational gradient descent \citep[a special case of VGD when $\mathcal{L}$ is linear;][]{liu2016stein}.
Again, we employed a minibatch of size $B = 256$ (for the experiments in \Cref{sec: ensembles}) or $B = 16$ (for the experiments in \Cref{sec: LoRA averaging}).

For all the experiments we report, the Gaussian kernel $k(\bm{\theta},\bm{\vartheta}) = \exp(-\ell^{-2} \| \bm{\theta} - \bm{\vartheta} \|^2)$ was used, with length-scale $\ell$ adaptively selected using the median heuristic \citep{garreau2017large}
$$
\ell^2 \equiv \ell(\bm{\theta}_{1:m})^2 = \frac{1}{2} \text{median}\{ \|\bm{\theta}_i - \bm{\theta}_j\|^2 : 1 \leq i < j \leq m \} .
$$
For the experiments in \Cref{sec: ensembles}, \ac{VGD} was run for $20{,}000$ optimisation steps with step size $\epsilon = 0.5$ on \textbf{Spiral} and $\epsilon = 0.1$ on \textbf{MNIST}.
For both datasets, \ac{VGD} was initialised from $4{,}000$ pre-training steps of \ac{MFLD} with step size $0.1$; we note that this initialisation is applied identically to \ac{VGD} and \ac{FVGD} but not to the \ac{MFLD} baseline, which is trained for the full $20{,}000$ steps.
For the experiments in \Cref{sec: LoRA averaging}, hyperparameters are as given in \Cref{tab:setup_hyperparams}.

\subsection{Functional VGD}
\label{app: implementation FVGD}

Similarly to the other baselines, we in practice employed a stochastic gradient momentum-based optimiser \citep[Adam;][]{kingma2014adam} with minibatch of size $B = 256$ for the experiments that we report in \Cref{sec: ensembles}.

For all experiments we used the minibatch-dependent kernel $k(f,f') = \exp( - \ell^{-2}\|\bm{\sigma}(f)^{1/2} - \bm{\sigma}(f')^{1/2}\|^2)$, where $\bm{\sigma}(f) \in \mathbb{R}^B$ is defined component-wise as
$$
\sigma(f)_i = \frac{1}{1 + e^{-f(\tilde{\mathbf{x}}_i)} }
$$
for each $f \in \mathfrak{F}$, where the $\tilde{\mathbf{x}}_i$ are the inputs sampled in the current minibatch.
This choice of $k$ can be interpreted as the Hellinger kernel applied to $\bm{\sigma}(f)$ and $\bm{\sigma}(f')$, and is therefore a valid kernel.
For our experiments in \Cref{sec: ensembles}, $f$ represents the model logits, and thus $\bm{\sigma}(f)$ is the vector containing the class-specific probabilities output by the model.
The length-scale $\ell = 0.1$ was fixed throughout.
\ac{FVGD} was run for $20{,}000$ steps with step size $\epsilon = 0.005$ on \textbf{Spiral} and $\epsilon = 0.001$ on \textbf{MNIST}, in each case initialised from $4,000$ pre-training steps of \ac{MFLD} with step size $0.1$.

\begin{remark}[Alternative kernels in \ac{FVGD}]
Several alternative choices for the kernel $k$ could be considered in future work.
    In a related context, a similarity measure with respect to the gradients of the model output, $ \nabla_1 f(\mathbf{x}, \bm{\theta})$, was considered in \citet{trinh2024input}. 
    As another possibility, to promote functional diversity without degrading uncertainty one could measure similarity with respect to $\left\{f(\tilde{\mathbf{x}}, \bm{\theta})\right\}_{\tilde{\mathbf{x}}\in \nu}$, where $\nu$ is a continuous distribution over the data domain, or a batch from an unlabelled validation set \citep[see Section 3.1.2 of][]{wang2018function}. 
\end{remark}

\subsection{Hardware and Software Used}

The classification experiments of \Cref{sec: ensembles} were run on a single workstation with the following specifications:

\begin{center}
\footnotesize
\begin{tabular}{@{}ll@{}}
\toprule
\textbf{Component} & \textbf{Specification} \\
\midrule
\multicolumn{2}{@{}l}{\textit{Hardware}} \\
CPU              & 12th Gen Intel(R) Core(TM) i9-12900K (24) @ 5.20 GHz \\
GPU             & NVIDIA GeForce RTX 4090 \\
Memory  & 125.48 GiB  \\
\addlinespace
\multicolumn{2}{@{}l}{\textit{Software}} \\
Python         &  3.12.12  \\
PyTorch          &  2.7.1+cu128  \\
OS           & Ubuntu noble 24.04 x86\_64 \\
Kernel          & Linux 6.17.0-35-generic \\
\bottomrule
\end{tabular}
\end{center}

\noindent The \ac{LLM} experiments of \Cref{sec: LoRA averaging} were run on a compute environment with the following specifications:

\begin{center}
\footnotesize
\begin{tabular}{@{}ll@{}}
\toprule
\textbf{Component} & \textbf{Specification} \\
\midrule
\multicolumn{2}{@{}l}{\textit{Hardware}} \\
GPU              & 8$\times$ NVIDIA H100 80\,GB SXM5 \\
CPU              & 128 cores \\
RAM              & 256\,GB \\
\addlinespace
\multicolumn{2}{@{}l}{\textit{Software}} \\
Python           & 3.10 \\
PyTorch          & 2.1.2+cu121 \\
Transformers     & 4.36.0 (Llama-3-8B); 4.45.0 (Llama-3.2-3B/1B) \\
Accelerate       & 0.25.0 (Llama-3-8B); 1.13.0 (Llama-3.2-3B/1B) \\
PEFT             & custom fork (commit released with code) \\
\bottomrule
\end{tabular}
\end{center}

\section{Experimental Protocol}
\label{app: protocol}

This appendix contains all of the details needed to reproduce the experimental results which we report.

\subsection{Ensemble Methods for Classification}
\label{app: ensemble protocol}

\paragraph{Outline of the Classification Task}

For these first experiments we consider classification tasks and employ the cross-entropy loss $L(\bm{y},\bm{p}) = - \sum_i y_i \log(p_i)$ where $\bm{y}$ is a one-hot vector indicating the true label and $\bm{p}$ is a vector of predicted class probabilities; we take $f(\mathbf{x},\bm{\theta}) = \mathrm{logit}(\bm{p})$ to be the output from the machine learning model.

\paragraph{Models Considered}

For the \textbf{Spirals} experiments, we used an ensemble of $m=10$ particles, where each particle is a fully connected neural network mapping the two-dimensional input to three output logits corresponding to the three spiral classes. Each particle has architecture $2 \rightarrow 2 \rightarrow 3$, with a hidden layer of width 2. For the single-network baseline, we used a larger fully connected neural network with architecture $2 \rightarrow 20 \rightarrow 3$.

For the \textbf{MNIST} experiments, we used an ensemble of $m=10$ particles, where each particle is a fully connected neural network with one hidden layer of width 4 and ReLU activation. Each network maps a flattened $28 \times 28$ MNIST image to 10 output logits corresponding to the digit classes. Thus, each particle has architecture $784 \rightarrow 4 \rightarrow 10$. The models were trained for 20,000 steps with batch size 256 and random seed 42. For the single-network baseline, we used a larger fully connected neural network with two hidden layers, each of width 16, giving the architecture $784 \rightarrow 16 \rightarrow 16 \rightarrow 10$.

\paragraph{Spirals}
The \textbf{Spirals} dataset was generated synthetically in two dimensions with three classes. For each class \(i \in \{0,1,2\}\), \(100\) observations were generated by taking radial coordinates \(r\) equally spaced on \([0,1]\) and angular coordinates on the interval \([4i,4(i+1)]\), perturbed by Gaussian noise with standard deviation \(0.2\). The Cartesian coordinates were then computed as \(x=r\sin(t)\) and \(y=r\cos(t)\), producing three noisy spiral arms, one for each class. This gives \(300\) observations in total, with balanced class labels.

We used a stratified train--test split with test ratio \(0.2\). The split was performed independently within each class, so that \(20\%\) of the observations from each spiral arm were assigned to the test set and the remaining \(80\%\) to the training set. Thus, under the default configuration, the training set contains \(240\) observations and the test set contains \(60\) observations, with equal class proportions preserved in both subsets.

\paragraph{MNIST}
For the \textbf{MNIST} experiments, we used the standard MNIST database of handwritten digits. The data were
obtained via \texttt{torchvision.datasets.MNIST}, which downloads the dataset to a local data
directory when required. MNIST comprises \(28 \times 28\) greyscale images of handwritten digits
from ten classes, labelled \(0,\ldots,9\). Each image was transformed into a tensor and normalised
using mean \(0.1307\) and standard deviation \(0.3081\), before being flattened into a
\(784\)-dimensional feature vector for the fully connected neural networks.

We retained the canonical MNIST split, using the provided training set for optimisation and the
provided test set for evaluation. Consequently, the training set contained \(60{,}000\) examples and
the test set contained \(10{,}000\) examples; no additional train--test resampling was applied.
Training was performed using mini-batches of size \(256\) sampled with shuffling, while test batches
were evaluated without shuffling. The MNIST dataset is derived from the original NIST digit
datasets; its copyright is held by Yann LeCun and Corinna Cortes, and it is distributed under the
Creative Commons Attribution--Share Alike 3.0 licence.

\subsection{LoRA Averaging for Fine-Tuning LLMs}
\label{app: LLM protocol}
This section provides full technical details for the \ac{LLM} experiments that we report in the main text.

\paragraph{Outline of the Reasoning Task}
Our experimental assessment was based on the Common\-sense-170K training set and the associated eight-task evaluation suite assembled by \citet{hu2023llmadapters}, covering skills such as binary question answering, physical and social commonsense, contextual completion, pronoun disambiguation, and science-style multiple-choice QA, while enabling a unified discriminative evaluation protocol.
The constituent benchmarks are BoolQ \citep{clark2019boolq}, PIQA \citep{bisk2020piqa}, SocialIQA \citep{sap2019socialiqa}, HellaSwag \citep{zellers2019hellaswag}, WinoGrande \citep{sakaguchi2021winogrande}, ARC-Easy and ARC-Challenge \citep{clark2018arc}, and OpenBookQA \citep{mihaylov2018openbookqa}.
Training used the Commonsense-170K set of ${\approx}170{,}000$ samples; evaluation used the held-out evaluation split of each constituent benchmark, with sizes as given in \Cref{tab:setup_benchmarks}.

Both the training set and the evaluation files were obtained from the LLM-Adapters repository, whose data are distributed under the Open Data Commons Attribution (ODC-By) licence and code under Apache-2.0; the constituent benchmarks retain their original licences, listed in \Cref{tab:setup_benchmarks}.
In accordance with the ODC-By licence, we acknowledge that this work contains information from the Commonsense-170K dataset of \citet{hu2023llmadapters} (\url{https://github.com/AGI-Edgerunners/LLM-Adapters}), which is made available under the ODC Attribution License.
The Llama-3-8B and Llama-3.2-3B/1B base models are used under the Meta Llama~3 and Llama~3.2 Community License Agreements.

\begin{table}[t!]
\caption{\textbf{Commonsense evaluation benchmarks.} Validation splits are used where official test labels are not public, following \citet{hu2023llmadapters}.}
\label{tab:setup_benchmarks}
\centering
\footnotesize
\begin{tabular}{@{}llrcl@{}}
\toprule
\textbf{Benchmark} & \textbf{Task type} & \textbf{Eval. size} & \textbf{Choices} & \textbf{Data licence} \\
\midrule
BoolQ         & binary QA               & 3,270  & 2    & CC BY-SA 3.0 \\
PIQA          & physical commonsense    & 1,838  & 2    & AFL 3.0 \\
SocialIQA     & social commonsense      & 1,954  & 3    & CC BY 4.0 \\
HellaSwag     & sentence completion     & 10,042 & 4    & MIT \\
WinoGrande    & pronoun disambiguation  & 1,267  & 2    & CC BY \\
ARC-Easy      & science QA              & 2,376  & 3--5 & CC BY-SA 4.0 \\
ARC-Challenge & science QA (hard)       & 1,172  & 3--5 & CC BY-SA 4.0 \\
OpenBookQA    & science QA (open-book)  & 500    & 4    & Apache 2.0 \\
\bottomrule
\end{tabular}
\end{table}
\FloatBarrier

For these experiments the training loss is the per-token cross-entropy over the answer span,
\[
L(\mathbf{y}_j , f(\mathbf{x}_j,\bm{\theta})) = - \sum_{t} \log p(y_{j,t} \mid y_{j,<t} , \mathbf{x}_j , \bm{\theta}) ,
\]
where $\mathbf{x}_j$ is the prompt, $\mathbf{y}_j$ the reference answer, and $p(\,\cdot \mid y_{j,<t}, \mathbf{x}_j , \bm{\theta})$ is the next-token distribution obtained by applying softmax to the logits $f(\mathbf{x}_j,\bm{\theta})$.
It is these logits that are averaged under \ac{OA}.

\begin{remark}[Evaluation on BoolQ]
Generation-based evaluation of ensemble models (greedy decoding over averaged logits) suffers from mode collapse on the BoolQ task, but only because greedy argmax over averaged logits introduces a majority-class bias for this boolean task.
The log-probability evaluation, which directly compares $P(\text{each choice} \mid \text{prompt})$, avoids this issue and was therefore used for BoolQ (only) in our assessment.
We stress that this concerns \emph{evaluation} only and does not alter the training objective: training minimises the per-token cross-entropy above for every task, and the two evaluation protocols differ only in how a prediction is extracted from the trained model at test time.
\end{remark}

\FloatBarrier

\paragraph{Models Considered}

To assess generality, we evaluate our methods on three models of different sizes within the same architecture family; Llama-3.2-1B \citep{meta2024llama3.2-1b}, Llama-3.2-3B \citep{meta2024llama3.2-3b} and Llama-3-8B \citep{meta2024llama3-8b}.
The specifications for these models are outlined in \Cref{tab:setup_models}.
All models use FP16 precision and a maximum token length of 256.
For all tasks except BoolQ, evaluation uses beam-search generation with beam width $4$ and regular-expression answer extraction, following the protocol of \citet{nitanda2025propagation}.
BoolQ instead uses the log-probability comparison described above.

\begin{table}[t!]
\caption{\textbf{Base model specifications.}
GQA = grouped query attention; the number of KV heads affects LoRA particle geometry on \texttt{k\_proj}/\texttt{v\_proj} layers.}
\label{tab:setup_models}
\centering
\footnotesize
\begin{tabular}{@{}l ccccc@{}}
\toprule
\textbf{Model} & \textbf{Params} & \textbf{Layers} & \textbf{Hidden} & \textbf{Attn heads} & \textbf{KV heads} \\
\midrule
Llama-3-8B     & 8.0B  & 32 & 4096 & 32 & 8 (GQA) \\
Llama-3.2-3B   & 3.2B  & 28 & 3072 & 24 & 8 (GQA) \\
Llama-3.2-1B   & 1.2B  & 16 & 2048 & 32 & 8 (GQA) \\
\bottomrule
\end{tabular}
\vspace{-4pt}
\end{table}

\FloatBarrier

\begin{table}[t!]
\caption{\textbf{Training hyperparameters.}
All methods share identical settings except the optimiser and temperature $\lambda$.}
\label{tab:setup_hyperparams}
\centering
\footnotesize
\begin{tabular}{@{}ll@{}}
\toprule
\textbf{Hyperparameter} & \textbf{Value} \\
\midrule
\multicolumn{2}{@{}l}{\textit{Training}} \\
Epochs                   & 3 (joint OA: 1) \\
Batch size               & 16 \\
Micro-batch size         & 16 \\
Learning rate            & $1 \times 10^{-4}$ \\
Weight decay             & 0.0 \\
Warmup steps             & 100 \\
Optimiser                & AdamW \\
Precision                & FP16  \\
Gradient checkpointing   & enabled \\
Max sequence length      & 256 tokens \\
\addlinespace
\multicolumn{2}{@{}l}{\textit{Data}} \\
Training data            & \texttt{commonsense\_170k.json} (170k samples) \\
Training steps/epoch     & $\sim$10,625 \\
Total training steps     & $\sim$31,875 (3 epochs) \\
\addlinespace
\multicolumn{2}{@{}l}{\textit{Method-specific}} \\
MFLD temperature         & $\lambda = 10^{-5}$ \citep{nitanda2025propagation} \\
VGD temperature          & $\lambda = 10^{-7}$ (Llama-3-8B) \\
VGD kernel               & RBF with median heuristic bandwidth \\
VGD kernel scope         & per LoRA layer, joint $(A,B)$ particle \\
\bottomrule
\end{tabular}
\vspace{-4pt}
\end{table}
\FloatBarrier

\section{Additional Experimental Results}
\label{app: additional}

This section contains several different sets of experimental results, to supplement those reported in the main text.

\subsection{Class Assignment for Spirals}

The learned class assignments for the \textbf{Spirals} dataset are displayed in \Cref{fig:spiral_all_boundary}.

\begin{figure*}[t!]
    \centering
    \begin{subfigure}[t]{0.19\linewidth}
        \centering
        \includegraphics[width=\linewidth]{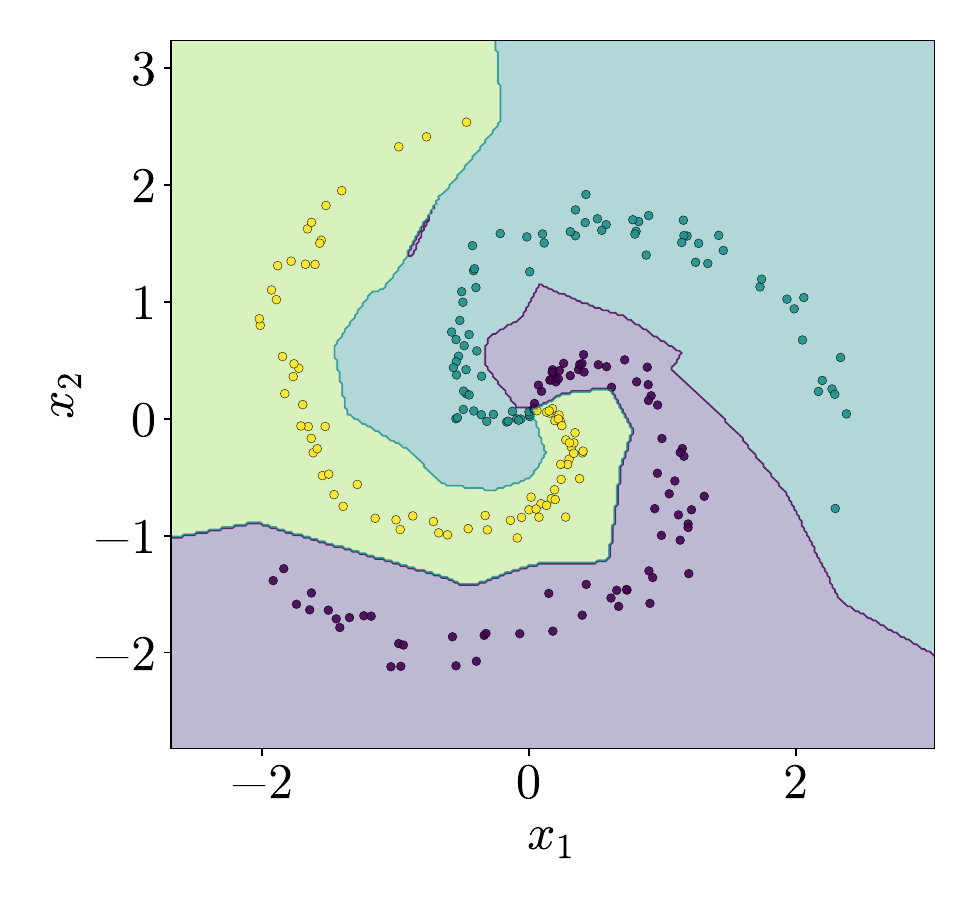}
        \caption{Single}
    \end{subfigure}
    \hfill
    \begin{subfigure}[t]{0.19\linewidth}
        \centering
        \includegraphics[width=\linewidth]{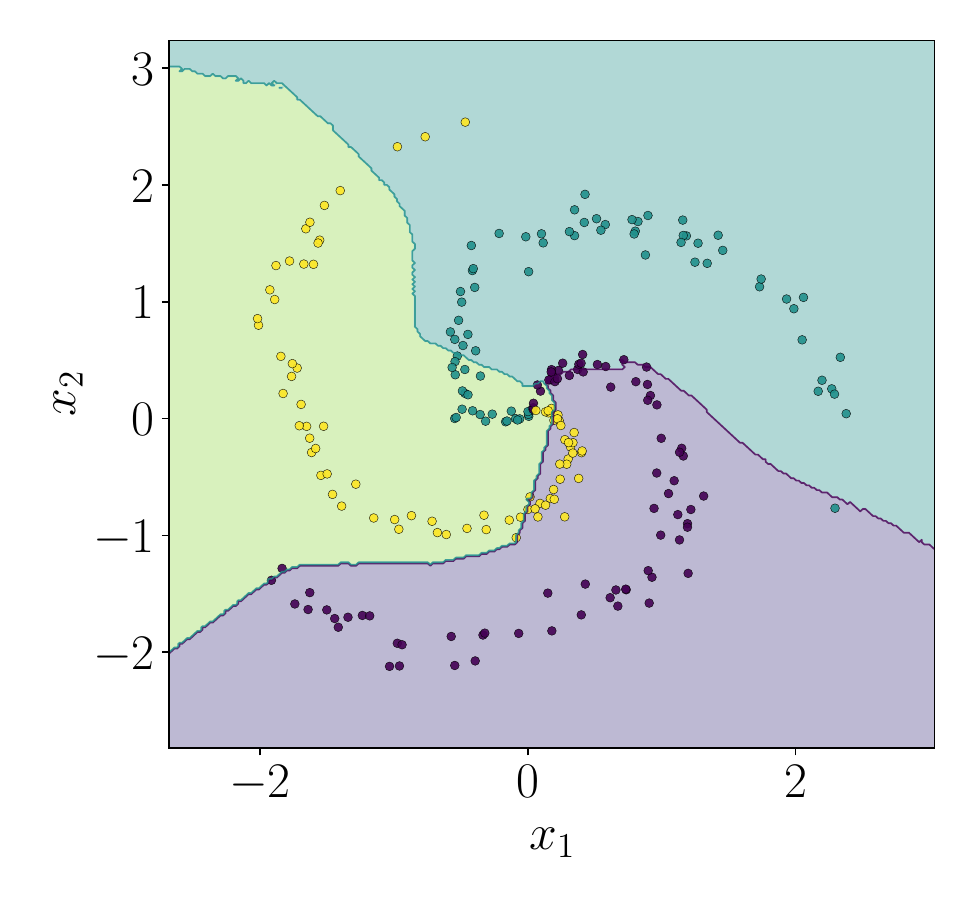}
        \caption{Independent}
    \end{subfigure}
    \hfill
    \begin{subfigure}[t]{0.19\linewidth}
        \centering
        \includegraphics[width=\linewidth]{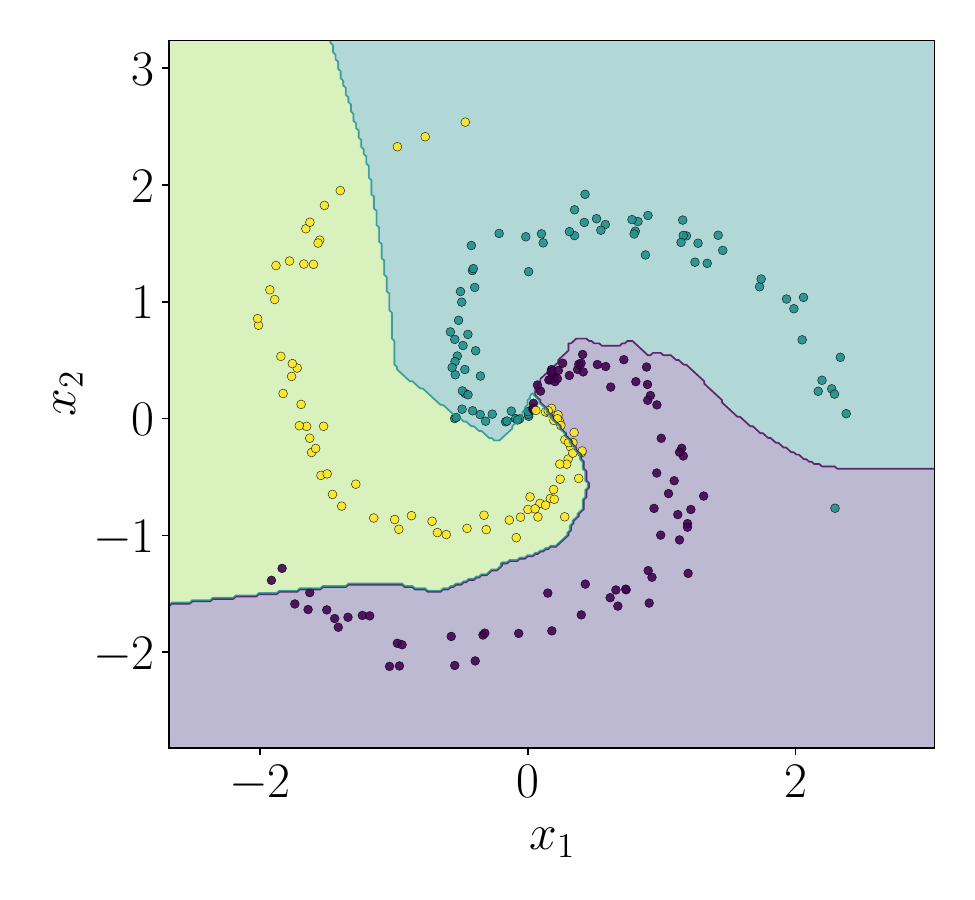}
        \caption{FVGD}
    \end{subfigure}
    \hfill
    \begin{subfigure}[t]{0.19\linewidth}
        \centering
        \includegraphics[width=\linewidth]{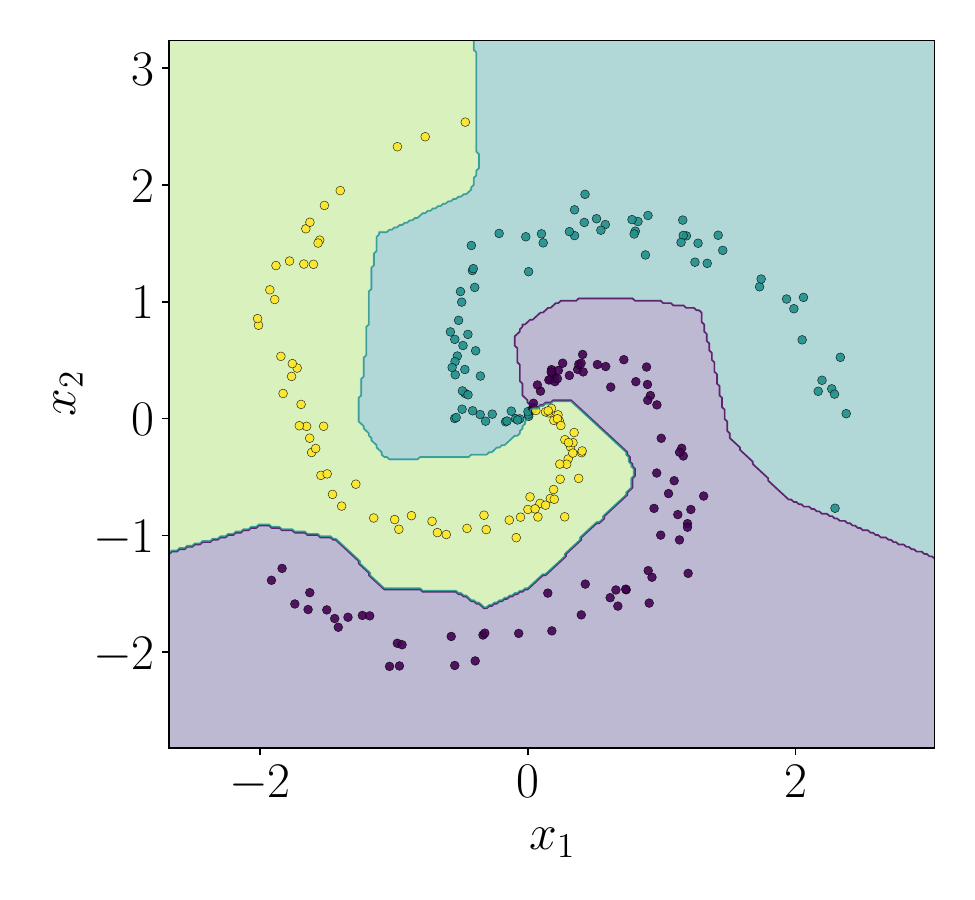}
        \caption{VGD}
    \end{subfigure}
    \hfill
    \begin{subfigure}[t]{0.19\linewidth}
        \centering
        \includegraphics[width=\linewidth]{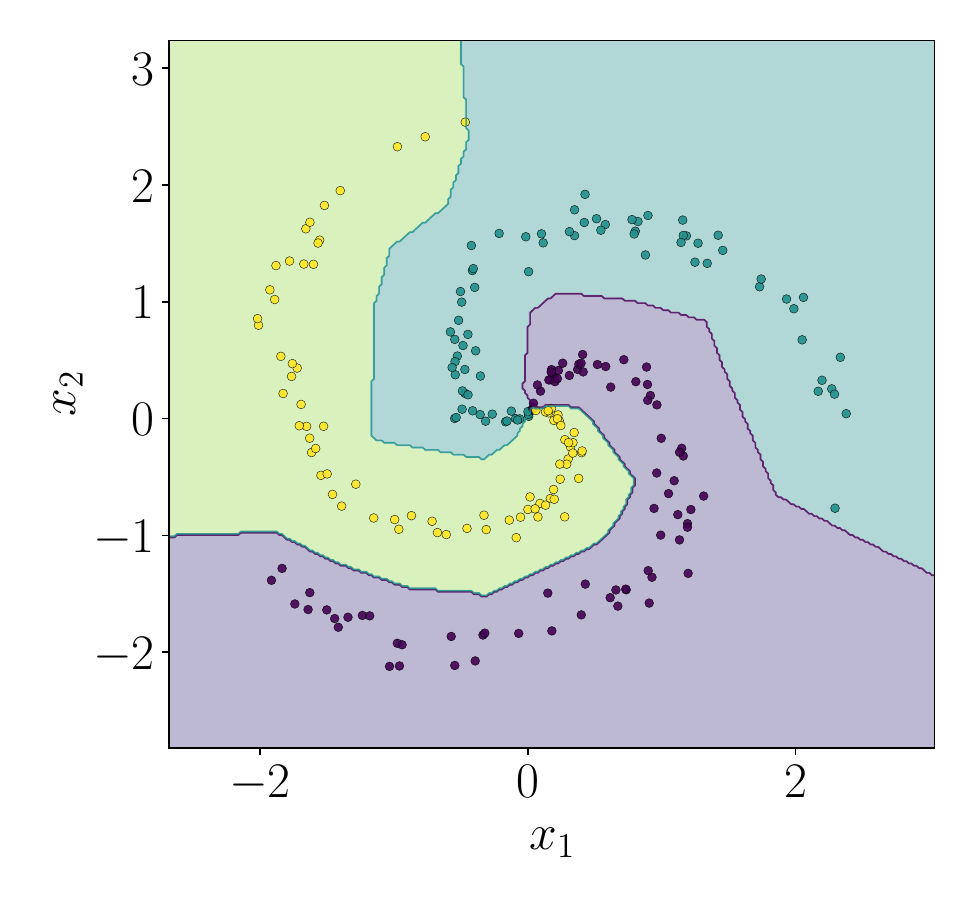}
        \caption{MFLD}
    \end{subfigure}

    \caption{\textbf{Spiral} dataset:  decision boundaries under a capacity-controlled setup. 
    Independent, MFLD, VGD, and FVGD all use an ensemble of $m=10$ MLPs with one hidden layer of width 2, while Single uses one hidden layer of width 20.
}
    \label{fig:spiral_all_boundary}
\end{figure*}

\subsection{A Closer Look at Performance on LoRA}
\label{app: visualise boxplot}

The rank-$32$ block of \Cref{tab:llm-main-best} in the main text is perhaps more easily understood through the visualisation in \Cref{fig:boxplot_seeds}.
This helps to make clear that, while the \emph{absolute differences in average performance} are modest, the \emph{distribution} of performance differs across methods; \ac{VGD} attains the highest median and the smallest standard error, with no isolated bad seed in these eight runs.
The raw data underlying these box plots are presented in \Cref{tab:per_seed_full}.

\begin{table}[!t]
  \centering
  \footnotesize
  \setlength{\tabcolsep}{4pt}
  \caption{\textbf{Per-seed accuracies (\%) on Llama-3-8B.}
  Rank-$32$ adapters are trained for $3$ epochs.
  Each row reports an independent run initialized with the indicated random seed.
  Means and standard errors (SE) over the $8$ runs are reported at the bottom of each block.
  Best mean results in each column are shown in bold.}
  \label{tab:per_seed_full}

  \begin{tabular}{@{}lccccccccc@{}}
    \toprule
    Seed & SIQA & PIQA & Wino & OBQA & ARC-c & ARC-e
    & BoolQ & Hella & \textbf{Avg} \\
    \midrule

    \multicolumn{10}{l}{\emph{Standard (MFLD with $\lambda=0$, vanilla AdamW)}} \\
    \quad $0$
    & $80.78$ & $88.61$ & $86.03$ & $87.86$
    & $79.12$ & $89.95$ & $74.77$ & $95.09$ & $85.28$ \\
    \quad $1$
    & $80.63$ & $87.55$ & $86.23$ & $85.63$
    & $78.81$ & $90.02$ & $74.14$ & $95.66$ & $84.83$ \\
    \quad $2$
    & $80.07$ & $88.81$ & $85.99$ & $85.25$
    & $79.45$ & $90.39$ & $74.43$ & $95.59$ & $85.00$ \\
    \quad $3$
    & $79.27$ & $88.57$ & $86.86$ & $84.53$
    & $79.72$ & $89.74$ & $76.02$ & $95.28$ & $85.00$ \\
    \quad $42$
    & $80.57$ & $88.18$ & $85.93$ & $86.68$
    & $78.68$ & $90.21$ & $76.27$ & $94.97$ & $85.19$ \\
    \quad $3407$
    & $80.11$ & $88.73$ & $86.35$ & $86.02$
    & $78.22$ & $90.89$ & $73.84$ & $95.78$ & $84.99$ \\
    \quad $2025$
    & $80.14$ & $88.64$ & $85.59$ & $83.88$
    & $77.68$ & $90.39$ & $74.32$ & $94.63$ & $84.41$ \\
    \quad $2026$
    & $80.27$ & $87.97$ & $87.25$ & $85.49$
    & $78.94$ & $90.48$ & $74.78$ & $95.08$ & $85.03$ \\
    \cmidrule(lr){1-10}
    \quad Mean
    & $80.23$ & $88.38$ & $86.28$ & $85.67$
    & $78.83$ & $90.26$ & $74.82$ & $95.26$ & $84.97$ \\
    \quad SE
    & $0.17$ & $0.16$ & $0.19$ & $0.44$
    & $0.23$ & $0.13$ & $0.31$ & $0.14$ & $0.09$ \\

    \midrule
    \multicolumn{10}{l}{\emph{MFLD, $\lambda=10^{-5}$}} \\
    \quad $0$
    & $80.30$ & $88.47$ & $86.41$ & $85.70$
    & $80.96$ & $90.22$ & $75.32$ & $95.28$ & $85.33$ \\
    \quad $1$
    & $80.66$ & $88.94$ & $85.85$ & $85.75$
    & $79.54$ & $91.56$ & $73.26$ & $95.34$ & $85.11$ \\
    \quad $2$
    & $80.84$ & $88.36$ & $86.47$ & $86.42$
    & $79.86$ & $90.65$ & $74.98$ & $94.52$ & $85.26$ \\
    \quad $3$
    & $80.49$ & $87.62$ & $86.27$ & $84.53$
    & $78.96$ & $89.21$ & $75.58$ & $95.90$ & $84.82$ \\
    \quad $42$
    & $80.46$ & $87.82$ & $86.28$ & $87.24$
    & $79.47$ & $89.58$ & $75.29$ & $95.37$ & $85.19$ \\
    \quad $3407$
    & $80.22$ & $87.33$ & $86.06$ & $86.34$
    & $80.36$ & $89.57$ & $75.34$ & $95.01$ & $85.03$ \\
    \quad $2025$
    & $79.80$ & $88.82$ & $86.75$ & $84.57$
    & $79.49$ & $90.71$ & $74.64$ & $95.53$ & $85.04$ \\
    \quad $2026$
    & $79.59$ & $88.89$ & $87.07$ & $88.17$
    & $80.00$ & $90.62$ & $75.16$ & $94.85$ & $85.54$ \\
    \cmidrule(lr){1-10}
    \quad Mean
    & $80.29$ & $88.28$ & $86.39$ & $86.09$
    & $\mathbf{79.83}$ & $90.26$ & $74.95$ & $95.22$ & $85.16$ \\
    \quad SE
    & $0.15$ & $0.22$ & $0.14$ & $0.44$
    & $0.22$ & $0.27$ & $0.26$ & $0.15$ & $0.08$ \\

    \midrule
    \multicolumn{10}{l}{\emph{VGD, $\lambda=10^{-7}$}} \\
    \quad $0$
    & $80.34$ & $88.88$ & $85.66$ & $86.81$
    & $79.76$ & $90.16$ & $75.52$ & $95.47$ & $85.33$ \\
    \quad $1$
    & $81.19$ & $89.31$ & $87.36$ & $85.08$
    & $80.44$ & $90.36$ & $75.27$ & $95.38$ & $85.55$ \\
    \quad $2$
    & $81.60$ & $89.24$ & $86.23$ & $86.84$
    & $79.55$ & $91.41$ & $74.58$ & $95.79$ & $85.66$ \\
    \quad $3$
    & $80.58$ & $87.97$ & $86.82$ & $85.83$
    & $79.72$ & $91.25$ & $75.26$ & $95.97$ & $85.43$ \\
    \quad $42$
    & $80.93$ & $89.02$ & $86.43$ & $87.18$
    & $79.21$ & $90.97$ & $74.03$ & $95.59$ & $85.42$ \\
    \quad $3407$
    & $81.31$ & $89.25$ & $86.39$ & $87.33$
    & $80.20$ & $90.74$ & $74.96$ & $94.99$ & $85.65$ \\
    \quad $2025$
    & $80.29$ & $88.71$ & $86.34$ & $87.78$
    & $80.10$ & $90.79$ & $75.37$ & $95.43$ & $85.60$ \\
    \quad $2026$
    & $80.62$ & $88.49$ & $86.57$ & $86.27$
    & $79.66$ & $90.33$ & $74.99$ & $94.91$ & $85.23$ \\
    \cmidrule(lr){1-10}
    \quad Mean
    & $\mathbf{80.86}$ & $\mathbf{88.86}$ & $\mathbf{86.48}$
    & $\mathbf{86.64}$ & $\mathbf{79.83}$ & $\mathbf{90.75}$
    & $\mathbf{75.00}$ & $\mathbf{95.44}$ & $\mathbf{85.48}$ \\
    \quad SE
    & $0.17$ & $0.16$ & $0.17$ & $0.31$
    & $0.14$ & $0.16$ & $0.17$ & $0.13$ & $0.06$ \\

    \bottomrule
  \end{tabular}
\end{table}

Adjusting for the difficulty of different tasks enables a clearer comparison of methods; this paired comparison is presented in \Cref{tab:llm-paired-bootstrap}.

\begin{figure}[t]
\centering
\includegraphics[width=0.6\textwidth]{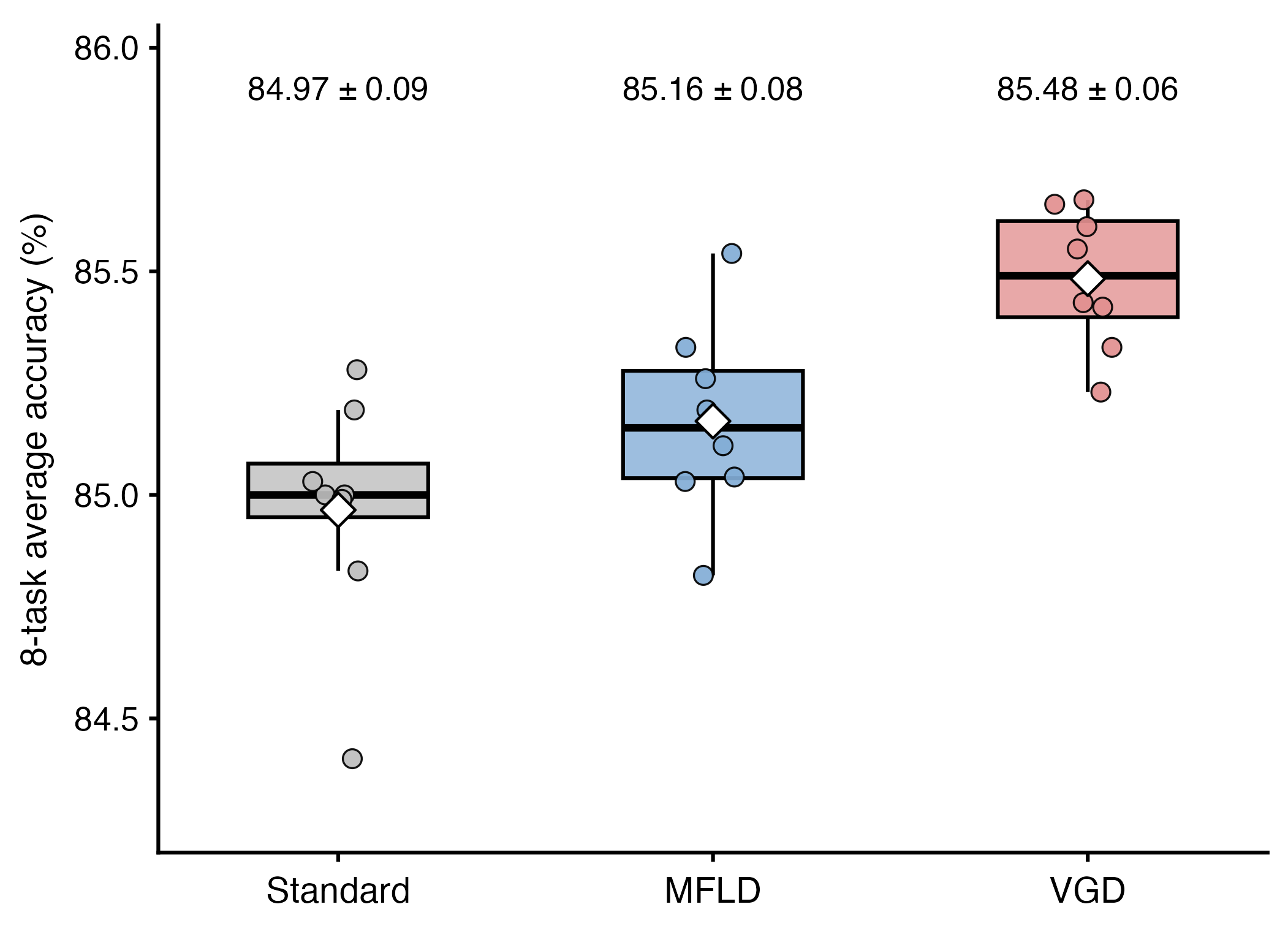}
\caption{\textbf{Distribution of 8-seed average accuracy on commonsense reasoning.} Each point represents one independent training run for a rank-$32$ LoRA adapter on Llama-3-8B, averaged across the eight commonsense benchmarks. Boxes indicate the interquartile range, centre lines indicate medians, whiskers extend to $1.5\times$ the interquartile range, and white diamonds indicate arithmetic means. Labels report mean $\pm$ standard error over the $8$ seeds.}
\label{fig:boxplot_seeds}
\end{figure}

\begin{table}[t]
\caption{\textbf{Paired task-level comparisons corresponding to
\Cref{tab:llm-main-best}.}
Each entry reports the mean paired accuracy difference in percentage points,
followed by a 95\% percentile bootstrap confidence interval obtained by
resampling the eight tasks with replacement.}
\label{tab:llm-paired-bootstrap}

\centering
\footnotesize
\setlength{\tabcolsep}{5pt}

\begin{tabular}{@{}lccc@{}}
\toprule
Setting
& $\Delta_{\mathrm{VGD}-\mathrm{Standard}}$
& $\Delta_{\mathrm{VGD}-\mathrm{MFLD}}$
& $\Delta_{\mathrm{MFLD}-\mathrm{Standard}}$ \\
\midrule

Single rank-$32$ adapter
& $+0.52\;[+0.31,+0.74]$
& $+0.32\;[+0.16,+0.48]$
& $+0.20\;[+0.01,+0.46]$ \\

Single rank-$256$ adapter 
& $+0.79\;[+0.38,+1.24]$
& $+1.71\;[+1.38,+2.04]$
& $-0.92\;[-1.20,-0.63]$ \\

PoC merge $\to$ rank-$256$ 
& $+0.31\;[+0.24,+0.39]$
& $+0.17\;[+0.04,+0.28]$
& $+0.14\;[+0.01,+0.28]$ \\

\bottomrule
\end{tabular}
\end{table}

\FloatBarrier

\FloatBarrier

\subsection{Varying \texorpdfstring{$\lambda$}{lambda}}
\label{app: vary lambda}

\paragraph{MNIST}

The effect of varying $\lambda$ over the \textbf{MNIST} experiment is reported in \Cref{tab:mnist_lambda_ablation}.

\begin{table*}[t!]
    \centering
    \footnotesize
    \caption{Ablation over the regularisation/repulsion parameter $\lambda$ on \textbf{MNIST}.}
    \label{tab:mnist_lambda_ablation}
    \begin{tabular}{llrrrr}
        \toprule
        Method & $\lambda$ & Train Loss & Test Loss & Train Acc. & Test Acc. \\
        \midrule
        \multirow{13}{*}{MFLD}
        & $0$                   & 0.2228 & 0.2338 & 0.938 & 0.932 \\
        & $10^{-6}$             & 0.2231 & 0.2337 & 0.938 & 0.931 \\
        & $3\times 10^{-6}$     & 0.2232 & 0.2338 & 0.938 & 0.931 \\
        & $10^{-5}$             & 0.2232 & 0.2337 & 0.938 & 0.931 \\
        & \textbf{$3\times 10^{-5}$}     & \textbf{0.2231} & \textbf{0.2336} & \textbf{0.938} & \textbf{0.932} \\
        & $10^{-4}$             & 0.2237 & 0.2337 & 0.938 & 0.931 \\
        & $10^{-3}$             & 0.2234 & 0.2337 & 0.938 & 0.931 \\
        & $10^{-2}$             & 0.2229 & 0.2336 & 0.938 & 0.931 \\
        & $10^{-1}$             & 0.2239 & 0.2338 & 0.938 & 0.931 \\
        & $1$                   & 0.2246 & 0.2345 & 0.934 & 0.930 \\
        & $10$                  & 0.2269 & 0.2430 & 0.941 & 0.929 \\
        & $20$                  & 0.2367 & 0.2527 & 0.938 & 0.926 \\
        \midrule
        \multirow{13}{*}{VGD}
        & $0$                   & 0.1143 & 0.1747 & 0.957 & 0.948 \\
        & $10^{-6}$             & 0.1143 & 0.1747 & 0.957 & 0.948 \\
        & $3\times 10^{-6}$     & 0.1143 & 0.1747 & 0.957 & 0.948 \\
        & $10^{-5}$             & 0.1143 & 0.1747 & 0.957 & 0.948 \\
        & $3\times 10^{-5}$     & 0.1142 & 0.1746 & 0.957 & 0.948 \\
        & $10^{-4}$             & 0.1143 & 0.1747 & 0.957 & 0.948 \\
        & $10^{-3}$             & 0.1140 & 0.1748 & 0.961 & 0.948 \\
        & $10^{-2}$             & 0.1148 & 0.1748 & 0.957 & 0.948 \\
        & $10^{-1}$             & 0.1145 & 0.1746 & 0.953 & 0.948 \\
        & \textbf{$1$}          & \textbf{0.1139} & \textbf{0.1746} & \textbf{0.961} & \textbf{0.948} \\
        & $10$                  & 0.1147 & 0.1747 & 0.961 & 0.948 \\
        & $20$                  & 0.1143 & 0.1747 & 0.957 & 0.948 \\
        \midrule
        \multirow{13}{*}{FVGD}
        & $0$                   & 0.3180 & 0.3766 & 0.906 & 0.921 \\
        & $10^{-6}$             & 0.3205 & 0.3707 & 0.918 & 0.919 \\
        & $3\times 10^{-6}$     & 0.3138 & 0.3604 & 0.918 & 0.921 \\
        & $10^{-5}$             & 0.3001 & 0.3438 & 0.906 & 0.918 \\
        & $1.7\times 10^{-5}$   & 0.3184 & 0.3483 & 0.902 & 0.920 \\
        & $3\times 10^{-5}$     & 0.3314 & 0.3616 & 0.891 & 0.921 \\
        & \textbf{$10^{-4}$}             & \textbf{0.3095} & \textbf{0.3612} & \textbf{0.918} & \textbf{0.923} \\
        & $10^{-3}$             & 0.3309 & 0.3679 & 0.895 & 0.919 \\
        & $10^{-2}$             & 0.3125 & 0.3514 & 0.902 & 0.914 \\
        & $10^{-1}$             & 1.1941 & 1.1557 & 0.824 & 0.868 \\
        & $1$                   & 2.2974 & 2.3010 & 0.117 & 0.114 \\
        & $10$                  & 2.2919 & 2.2959 & 0.117 & 0.114 \\
        & $20$                  & 2.2962 & 2.3003 & 0.117 & 0.114 \\
        \bottomrule
    \end{tabular}
\end{table*}
    
\FloatBarrier

\paragraph{LoRA Averaging for Fine-Tuning LLMs}

\Cref{tab:full_lambda_sweep_appendix} reports a single-seed $\lambda$ sweep on Llama-3-8B.
The best observed average occurs at $\lambda=10^{-7}$ for rank $32$, $3.2\times10^{-7}$ for rank $128$, and $10^{-6}$ for rank $256$.
Larger values degrade performance sharply.
At rank $32$, for example, $\lambda=10^{-4}$ reduces average accuracy from $85.42\%$ to $74.28\%$, a decrease of $11.14$ percentage points.

\Cref{tab:llm_lambda_model_size} gives the corresponding comparison across model sizes.
The best observed candidates are $10^{-7}$ for 8B, $10^{-6}$ for 3B, and $10^{-5}$ for 1B.
The 1B model is particularly sensitive, reaching $47.81\%$, $60.92\%$, and $66.34\%$ at the three candidate values, respectively.
This pattern suggests model-dependent sensitivity, although a single-seed sweep does not establish a scaling law.

\begin{table}[t!]
\centering
\footnotesize
\setlength{\tabcolsep}{3pt}
\caption{\textbf{Full $\lambda$ sweep on Llama-3-8B} (accuracy, \%).
Results for single seed are reported after 3 epochs.
Averages are computed before rounding.
Best results in each column within each rank block are shown in bold.}
\label{tab:full_lambda_sweep_appendix}

\begin{tabular}{@{}ll*{9}{c}@{}}
\toprule
Method & $\lambda$ & SIQA & PIQA & Wino & OBQA
& ARC-c & ARC-e & BoolQ & Hella & Avg \\
\midrule

\multicolumn{11}{l}{\emph{Rank 32}} \\
Standard & $0$
& 80.57 & 88.18 & 85.93 & 86.68
& 78.68 & 90.21 & \textbf{76.27} & 94.97 & 85.19 \\

MFLD & $10^{-5}$
& 80.46 & 87.82 & 86.28 & \textbf{87.24}
& 79.47 & 89.58 & 75.29 & 95.37 & 85.19 \\

VGD & $0$
& 80.14 & 88.21 & 86.17 & 86.01
& 79.02 & 90.36 & 74.81 & 95.41 & 85.02 \\

VGD & $10^{-7}$
& \textbf{80.93} & \textbf{89.02} & 86.43 & 87.18
& 79.21 & \textbf{90.97} & 74.03 & \textbf{95.59}
& \textbf{85.42} \\

VGD & $3.2{\times}10^{-7}$
& 80.72 & 88.66 & 86.51 & 86.23
& 79.44 & 90.61 & 75.06 & 95.23 & 85.31 \\

VGD & $10^{-6}$
& 80.68 & 88.51 & \textbf{86.62} & 85.94
& \textbf{79.56} & 90.72 & 74.88 & 95.04 & 85.24 \\

VGD & $10^{-5}$
& 80.08 & 87.72 & 85.91 & 85.82
& 78.74 & 89.94 & 74.11 & 94.62 & 84.62 \\

VGD & $3.2{\times}10^{-5}$
& 76.44 & 84.27 & 81.19 & 82.06
& 74.14 & 86.21 & 70.46 & 89.53 & 80.54 \\

VGD & $10^{-4}$
& 70.83 & 77.18 & 74.29 & 73.04
& 67.43 & 83.67 & 63.19 & 84.63 & 74.28 \\

VGD & $3.2{\times}10^{-4}$
& 58.42 & 72.11 & 69.88 & 65.37
& 56.18 & 68.42 & 61.04 & 73.29 & 65.59 \\

\midrule
\multicolumn{11}{l}{\emph{Rank 128}} \\
Standard & $0$
& 80.77 & 87.81 & 85.48 & 83.89
& 76.76 & 89.07 & 74.00 & 94.04 & 83.98 \\

MFLD & $10^{-5}$
& 80.33 & 87.54 & 85.11 & 83.25
& 76.38 & 88.42 & 73.63 & 93.77 & 83.55 \\

VGD & $0$
& 80.85 & 88.03 & 85.64 & 84.12
& 76.92 & 89.18 & 74.21 & 94.11 & 84.13 \\

VGD & $10^{-7}$
& 81.04 & 88.18 & 85.82 & 84.39
& 77.18 & 89.52 & 74.37 & 94.01 & 84.31 \\

VGD & $3.2{\times}10^{-7}$
& \textbf{81.46} & \textbf{88.36} & \textbf{85.97}
& \textbf{84.70} & \textbf{77.41} & \textbf{89.81}
& \textbf{74.48} & \textbf{94.16} & \textbf{84.54} \\

VGD & $10^{-6}$
& 80.92 & 88.06 & 85.86 & 84.24
& 77.09 & 89.64 & 74.11 & 93.83 & 84.22 \\

VGD & $10^{-5}$
& 80.04 & 87.11 & 84.93 & 83.61
& 75.84 & 88.56 & 72.92 & 93.28 & 83.29 \\

VGD & $3.2{\times}10^{-5}$
& 74.49 & 80.39 & 76.18 & 76.17
& 72.64 & 81.44 & 67.12 & 88.58 & 77.13 \\

VGD & $10^{-4}$
& 70.19 & 71.76 & 72.54 & 72.10
& 57.09 & 78.79 & 54.76 & 83.14 & 70.05 \\

VGD & $3.2{\times}10^{-4}$
& 61.24 & 68.77 & 64.10 & 62.48
& 50.31 & 69.52 & 49.68 & 67.93 & 61.75 \\

\midrule
\multicolumn{11}{l}{\emph{Rank 256}} \\
Standard & $0$
& 80.23 & 86.99 & 84.39 & 81.54
& 74.85 & 87.94 & 73.40 & 92.69 & 82.75 \\

MFLD & $10^{-5}$
& 78.79 & 86.06 & 84.24 & 80.98
& 73.38 & 86.94 & 72.36 & 91.89 & 81.83 \\

VGD & $0$
& 80.45 & 87.04 & 84.31 & 81.89
& 74.91 & 88.14 & 73.48 & 93.12 & 82.92 \\

VGD & $10^{-7}$
& 80.72 & 87.28 & 84.57 & 82.21
& 75.02 & 88.22 & 73.69 & 92.73 & 83.05 \\

VGD & $3.2{\times}10^{-7}$
& \textbf{80.88} & 87.51 & 84.83 & 82.58
& 75.37 & 88.46 & \textbf{73.76} & 93.04 & 83.30 \\

VGD & $10^{-6}$
& 80.60 & \textbf{87.69} & \textbf{85.87} & \textbf{83.50}
& \textbf{75.65} & \textbf{88.51} & 73.18
& \textbf{93.34} & \textbf{83.54} \\

VGD & $10^{-5}$
& 76.21 & 82.48 & 81.57 & 77.09
& 71.06 & 83.01 & 68.47 & 87.84 & 78.47 \\

VGD & $3.2{\times}10^{-5}$
& 64.12 & 69.15 & 67.88 & 66.31
& 62.17 & 76.20 & 61.83 & 76.42 & 68.01 \\

VGD & $10^{-4}$
& 54.09 & 70.44 & 60.61 & 64.27
& 57.91 & 65.37 & 45.18 & 63.84 & 60.21 \\

VGD & $3.2{\times}10^{-4}$
& 46.84 & 61.29 & 52.43 & 55.18
& 44.77 & 57.63 & 41.28 & 52.71 & 51.52 \\

\bottomrule
\end{tabular}
\end{table}

\begin{table}[t]
\centering
\footnotesize
\caption{\textbf{VGD accuracy across model sizes and candidate
$\lambda$ values}. Best results for each model are shown in bold.}
\label{tab:llm_lambda_model_size}

\begin{tabular}{@{}lccc@{}}
\toprule
$\lambda$
& \textbf{Llama-3-8B}
& \textbf{Llama-3.2-3B}
& \textbf{Llama-3.2-1B} \\
\midrule

$10^{-7}$
& \textbf{85.42}
& 80.73
& 47.81 \\

$10^{-6}$
& 85.24
& \textbf{80.94}
& 60.92 \\

$10^{-5}$
& 84.62
& 80.43
& \textbf{66.34} \\

\bottomrule
\end{tabular}
\end{table}

\subsection{Kernel Choice for LoRA Averaging}
\label{app: kernel ablation}

A key implementation choice is \emph{how} to define a particle in the \ac{VGD} kernel.
Each LoRA layer has two matrices $\mathbf{A} \in \mathbb{R}^{r \times d_{\mathrm{in}}}$ and $\mathbf{B} \in \mathbb{R}^{d_{\mathrm{out}} \times r}$, giving $r$ rank-1 components.
A perhaps natural choice is to treat $\mathbf{A}[i,:]$ and $\mathbf{B}[:,i]$ as \textbf{separate}, using a kernel of the form
$$
k_A(\mathbf{A}[i,:] , \tilde{\mathbf{A}}[i,:]) + k_B(\mathbf{B}[i,:] , \tilde{\mathbf{B}}[i,:]) .
$$
The separated kernel discards the cross-terms between $A$ and $B$.
On the other hand, a more cautious choice could be to \emph{concatenate} the two components into a single vector $\bm{\theta}_i = (\mathbf{A}[i,:] \| \mathbf{B}[:,i]) \in \mathbb{R}^{d_{\mathrm{in}} + d_{\mathrm{out}}}$ and compute a \textbf{joint} kernel, of the form
$$
k\left( (\mathbf{A}[i,:] \| \mathbf{B}[:,i])  , (\tilde{\mathbf{A}}[i,:] \| \tilde{\mathbf{B}}[:,i])  \right) .
$$
Such a joint kernel arguably more correctly proxies the functional distance $\|\Delta\mathbf{W}_i - \Delta\mathbf{W}_j\|_F$ through $\|\bm{\theta}_i - \bm{\theta}_j\|^2 = \|\mathbf{a}_i - \mathbf{a}_j\|^2 + \|\mathbf{b}_i - \mathbf{b}_j\|^2$.
Both formulations have identical computational cost $O(r^2 D)$ per layer per step; the separated kernel is not a cheaper approximation but a weaker kernel at the same cost.

\Cref{tab:particle_ablation} quantifies the impact of this kernel choice across three model sizes.
In all cases, the joint kernel offered improved performance compared to the separate kernel.
The degradation from separated particles scales inversely with model size, likely because smaller models have less capacity to absorb misdirected repulsion.
All subsequent experiments used the joint kernel.

\begin{table}[t!]
\caption{\textbf{Particle-definition ablation on the final evaluation split}
(accuracy, \%).
All runs use single seed and rank-$32$ adapters.
The joint kernel acts on
$(\mathbf{A}[i,:]\mid\mathbf{B}[:,i])$,
while the separate kernel sums the kernels over
$\mathbf{A}$ and $\mathbf{B}$.
$\Delta_{\mathrm{MFLD}}$ is the change in eight-task average accuracy
relative to MFLD.
Best task and average accuracies within each model block are shown in bold.}
\label{tab:particle_ablation}

\centering
\footnotesize
\setlength{\tabcolsep}{3pt}

\resizebox{\linewidth}{!}{%
\begin{tabular}{@{}lll*{10}{c}@{}}
\toprule
\textbf{Model}
& \textbf{Kernel}
& $\boldsymbol{\lambda}$
& \textbf{SIQA}
& \textbf{PIQA}
& \textbf{Wino}
& \textbf{OBQA}
& \textbf{ARC-c}
& \textbf{ARC-e}
& \textbf{BoolQ}
& \textbf{Hella}
& \textbf{Avg}
& $\boldsymbol{\Delta}_{\mathrm{MFLD}}$ \\
\midrule

\multirow{3}{*}{Llama-3-8B}
& MFLD
& $10^{-5}$
& 80.46 & 87.82 & 86.28 & \textbf{87.24}
& \textbf{79.47} & 89.58 & \textbf{75.29} & 95.37
& 85.19 & \textemdash \\

& VGD, joint
& $10^{-7}$
& \textbf{80.93} & \textbf{89.02} & \textbf{86.43} & 87.18
& 79.21 & \textbf{90.97} & 74.03 & \textbf{95.59}
& \textbf{85.42} & $+0.23$ \\

& VGD, separate
& $10^{-7}$
& 79.61 & 87.46 & 84.92 & 85.66
& 77.89 & 89.47 & 72.31 & 94.38
& 83.96 & $-1.23$ \\

\midrule

\multirow{3}{*}{Llama-3.2-3B}
& MFLD
& $10^{-5}$
& 78.30 & 84.71 & 80.58 & 79.20
& 73.89 & 85.31 & 70.24 & 91.87
& 80.51 & \textemdash \\

& VGD, joint
& $10^{-6}$
& \textbf{78.84} & \textbf{85.17} & \textbf{80.91} & \textbf{80.14}
& \textbf{74.02} & \textbf{85.63} & \textbf{70.71} & \textbf{92.12}
& \textbf{80.94} & $+0.43$ \\

& VGD, separate
& $10^{-6}$
& 74.12 & 80.36 & 75.94 & 73.07
& 64.82 & 77.31 & 62.88 & 84.35
& 74.11 & $-6.41$ \\

\midrule

\multirow{3}{*}{Llama-3.2-1B}
& MFLD
& $10^{-5}$
& 69.96 & 75.41 & 67.72 & 66.60
& 51.37 & 69.78 & 62.29 & 61.19
& 65.54 & \textemdash \\

& VGD, joint
& $10^{-5}$
& \textbf{70.88} & \textbf{76.20} & \textbf{68.27} & \textbf{67.42}
& \textbf{52.04} & \textbf{70.61} & \textbf{62.58} & \textbf{62.68}
& \textbf{66.34} & $+0.80$ \\

& VGD, separate
& $10^{-5}$
& 62.31 & 65.74 & 58.26 & 55.43
& 42.18 & 60.47 & 51.06 & 47.82
& 55.41 & $-10.13$ \\

\bottomrule
\end{tabular}%
}
\end{table}

\FloatBarrier

\subsection{Results for Other LLMs}
\label{app: choice of LLM}

To assess generality, we additionally evaluate on two smaller models within the same architecture family; LLaMA-3.2-1B \citep{meta2024llama3.2-1b} and LLaMA-3.2-3B \citep{meta2024llama3.2-3b}.
\Cref{tab:finding2_single} reports results for a single seed.
In this single-seed comparison, \ac{VGD} matches or out-performs \ac{MFLD} on all three model sizes, with the largest margin on the smallest model; we caution that the margins are small and would benefit from multi-seed replication.

\begin{table}[t]
\caption{\textbf{VGD vs MFLD: single rank-$32$ adapter across model sizes}
(single seed).
MFLD uses the fixed $\lambda=10^{-5}$.
For VGD, we report the best observed candidate for each model in
\Cref{tab:llm_lambda_model_size}.
In this single-seed comparison, \ac{VGD} achieves higher average accuracy
than \ac{MFLD} at every model size, with the gap increasing as model size
decreases.}
\label{tab:finding2_single}

\centering
\setlength{\tabcolsep}{3.2pt}
\footnotesize

\resizebox{\textwidth}{!}{%
\begin{tabular}{@{}lll*{9}{c}@{}}
\toprule
\textbf{Model}
& \textbf{Method}
& $\boldsymbol{\lambda}$
& \textbf{SIQA}
& \textbf{PIQA}
& \textbf{Wino}
& \textbf{OBQA}
& \textbf{ARC-c}
& \textbf{ARC-e}
& \textbf{BoolQ}
& \textbf{Hella}
& \textbf{Avg} \\
\midrule

\multirow{2}{*}{\scriptsize Llama-3-8B}
& MFLD
& $10^{-5}$
& 80.46
& 87.82
& 86.28
& \textbf{87.24}
& \textbf{79.47}
& 89.58
& \textbf{75.29}
& 95.37
& 85.19 \\

& \textbf{VGD}
& $10^{-7}$
& \textbf{80.93}
& \textbf{89.02}
& \textbf{86.43}
& 87.18
& 79.21
& \textbf{90.97}
& 74.03
& \textbf{95.59}
& \textbf{85.42} \\

\midrule

\multirow{2}{*}{\scriptsize Llama-3.2-3B}
& MFLD
& $10^{-5}$
& 78.30
& 84.71
& 80.58
& 79.20
& 73.89
& 85.31
& 70.24
& 91.87
& 80.51 \\

& \textbf{VGD}
& $10^{-6}$
& \textbf{78.84}
& \textbf{85.17}
& \textbf{80.91}
& \textbf{80.14}
& \textbf{74.02}
& \textbf{85.63}
& \textbf{70.71}
& \textbf{92.12}
& \textbf{80.94} \\

\midrule

\multirow{2}{*}{\scriptsize Llama-3.2-1B}
& MFLD
& $10^{-5}$
& 69.96
& 75.41
& 67.72
& 66.60
& 51.37
& 69.78
& 62.29
& 61.19
& 65.54 \\

& \textbf{VGD}
& $10^{-5}$
& \textbf{70.88}
& \textbf{76.20}
& \textbf{68.27}
& \textbf{67.42}
& \textbf{52.04}
& \textbf{70.61}
& \textbf{62.58}
& \textbf{62.68}
& \textbf{66.34} \\

\bottomrule
\end{tabular}%
}

\vspace{-4pt}
\end{table}

\FloatBarrier

\subsection{Computational Complexity and Timings}
\label{app: timing}

\paragraph{Time and space complexity.}
Let $p$ denote the dimension of a single particle $\bm{\theta}_i$, $B$ the minibatch size, $e$ the output dimension of the model, $T$ the number of optimisation steps, and $c_B$ the cost of one forward--backward pass of a single model on a minibatch of size $B$.
Each step of \ac{MFLD} costs $O(m c_B + mp)$ in time, being one forward--backward pass per particle plus the parameter update and the Gaussian perturbation; because of minibatching, this does not depend on the size $n$ of the dataset.
\ac{VGD} additionally requires the $m \times m$ kernel matrix, the median-heuristic length-scale, and the two products appearing in \eqref{eq: vgd ode}, at total cost $O(m^2 p)$ per step, for an overall per-step cost of $O(m c_B + m^2 p)$.
(For the Gaussian kernel, $(\nabla_2 \mathbf{K}) \mathbf{1}$ can be evaluated without forming an $m \times m \times p$ array, since $\sum_{j} (\nabla_2 k)(\bm{\theta}_i , \bm{\theta}_j) = 2 \ell^{-2} \sum_{j} k(\bm{\theta}_i , \bm{\theta}_j) (\bm{\theta}_i - \bm{\theta}_j)$.)
The kernel term is of lower order than the gradient computation whenever $m p \lesssim c_B$, as is the case in our \ac{LoRA} experiments where $c_B$ is dominated by the frozen base model.
\ac{FVGD} evaluates the kernel on the minibatch outputs, which lie in $\mathbb{R}^{Be}$, at cost $O(m^2 B e)$ that is independent of $p$, and the pull-back in \eqref{eq: proj_fvgd} is a vector--Jacobian product costing $O(c_B)$ per particle, so that the per-step cost is $O(m c_B + m^2 B e)$.
In each case the total time is $T$ times the per-step cost.
For space, all three methods store the $m$ particles, i.e. $O(mp)$ numbers (plus optimiser state, a constant multiple of this), with \ac{VGD} adding $O(m^2)$ for the kernel matrix and \ac{FVGD} adding $O(m^2 + mBe)$; the additional memory required by the simple strategy for computing variational gradients is discussed in \Cref{app: var grad}.
At test time, an ensemble of $m$ models requires $m$ forward passes, which can be run in parallel, whereas a \ac{LoRA} averaged adapter can be merged into a single adapter of rank $R = mr$ and therefore has the test-time cost of a standard rank-$R$ \ac{LoRA}.

\paragraph{Timings.}
The use of \ac{VGD} in place of \ac{MFLD} adds a kernel evaluation of cost $O(m^2 D)$ per layer per step, against a forward--backward pass whose cost is $O(mD)$ in the adapter parameters but is dominated by the frozen base model. The resulting overhead is small, as reported in \Cref{tab:compute_cost}.

\begin{table}[t!]
\caption{\textbf{Computational cost} on Llama-3-8B (rank 32, 3 epochs, 8$\times$H100).
The use of \ac{VGD} adds only ${\sim}6\%$ wall-clock time compared to \ac{MFLD}.}
\label{tab:compute_cost}
\centering
\footnotesize
\begin{tabular}{@{}l l l c c@{}}
\toprule
\textbf{Method} & \textbf{Per-step cost} & \textbf{Mechanism} & \textbf{Wall-clock} & \textbf{Overhead} \\
\midrule
Standard LoRA & $O(RD)$ & gradient only & ${\sim}5$h\,30m & \textemdash \\
MFLD & $O(mrD)$ & $+$\,\texttt{randn} & ${\sim}5$h\,30m & ${\sim}0\%$ \\
VGD & $O(mrD+m^2rD)$ & $+$\,kernel repulsion & ${\sim}5$h\,50m & ${\sim}6\%$ \\
\bottomrule
\end{tabular}

\vspace{4pt}
{\footnotesize
$m=32$, $r=1$, $R=mr=32$, and
$D=d_{\mathrm{in}}+d_{\mathrm{out}}\approx11\text{k}$.}
\vspace{-4pt}
\end{table}

\FloatBarrier

\subsection{FVGD Diagnosis}
\label{app: fvgd diagnosis}

\begin{figure}[t!]
    \hfill
    \begin{subfigure}[t]{0.45\linewidth}
        \centering
        \includegraphics[width=\linewidth]{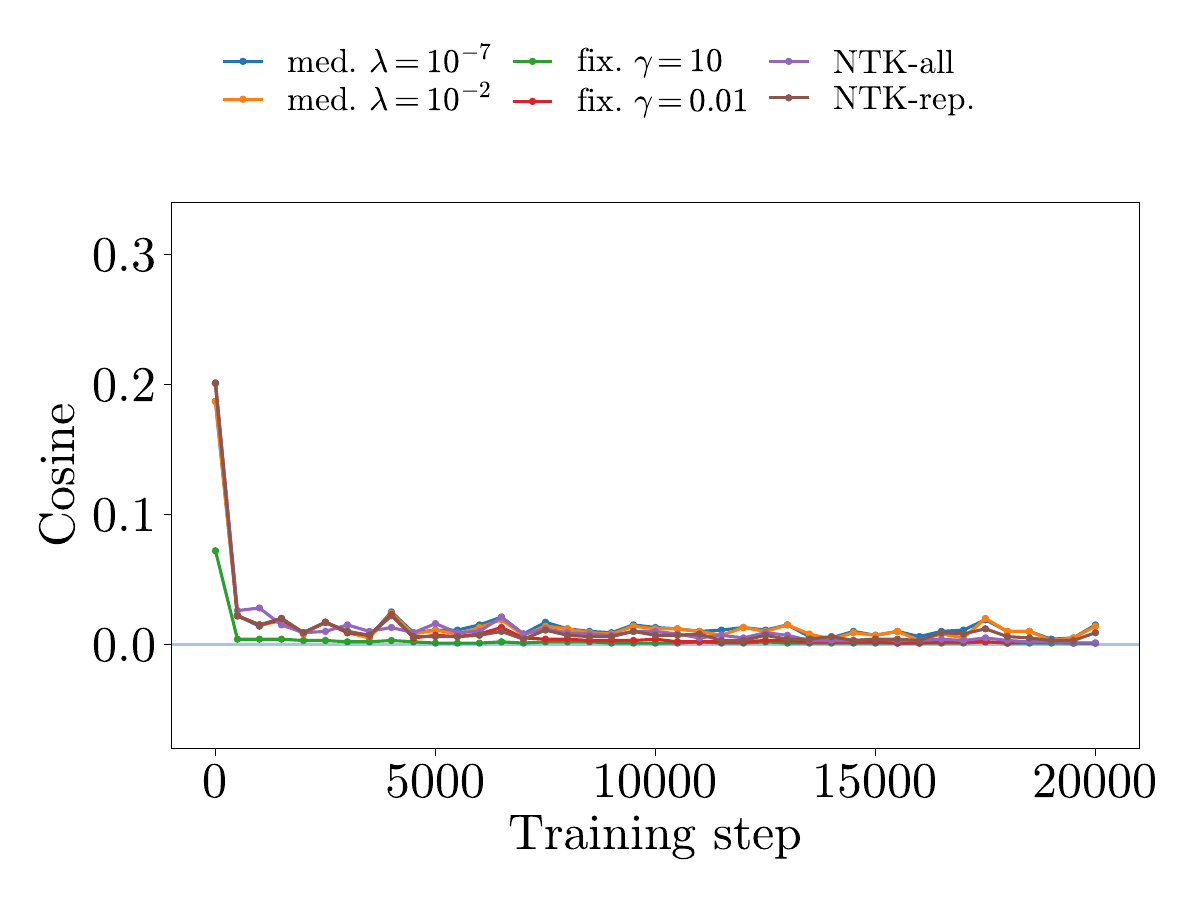}
        \caption{Cosine similarity \eqref{eq: cosine similarity} between the Jacobian-pulled-back change and the target functional direction}
        \label{fig:fvgd_diagnosis_pullback_alignment}
    \end{subfigure}
    \hfill
    \begin{subfigure}[t]{0.45\linewidth}
        \centering
        \includegraphics[width=\linewidth]{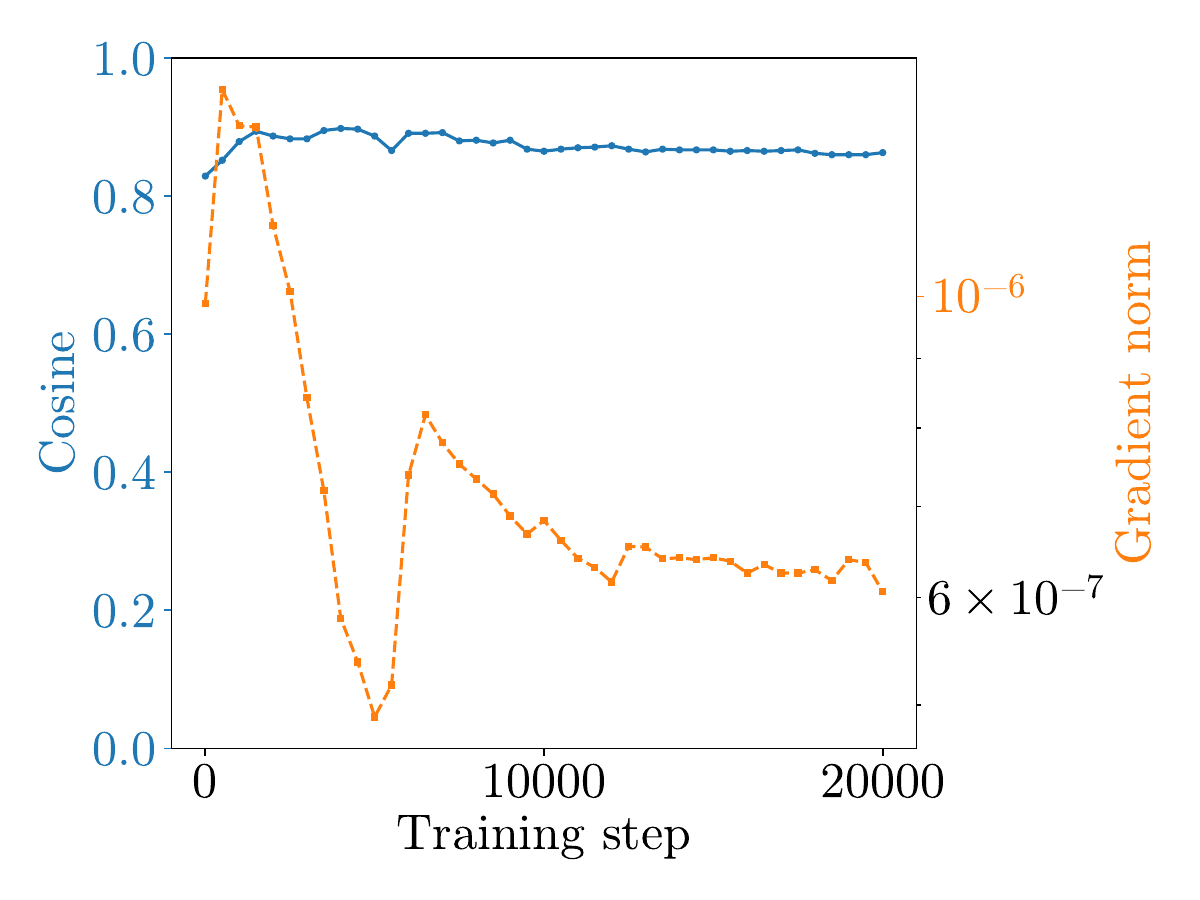}
        \caption{NTK repulsion alignment and its negligible parameter-space norm}
        \label{fig:fvgd_diagnosis_ntk_repulsion_failure}
    \end{subfigure}
    \caption{Understanding the poor performance of \ac{FVGD}.}
\end{figure}

The aim of this appendix is to understand the underwhelming performance of \ac{FVGD} in the experiments reported in \Cref{sec: ensembles}.
The key issue is whether a sensible parameter update is induced from the proposed function-space direction: 
$$
 \bm{\theta}_i^{t+1} = \bm{\theta}_i^t + \epsilon J_i^\top\bm{\phi}_i, \qquad J_i = D_{\bm{\theta}}F_n(\bm{\theta}_i^t),
$$
then, to first order,
$$
F_n(\bm{\theta}_i^{t+1}) -  F_n(\bm{\theta}_i^t)  =  \epsilon J_iJ_i^\top\bm{\phi}_i + O(\epsilon^2).
$$
This shows that the induced functional update is warped by a parameterisation-dependent map $J_iJ_i^\top$, and the \ac{FVGD} dynamics in function space are not independent of the model parameterisation.
The same issue is discussed in \citet[][Section~3.1.3]{wang2018function}, where it is argued that the Jacobian is constant in $\theta$ (appealing to the neural tangent kernel infinite-width heuristic of \citealp{jacot2018neural}).
More precisely, the approximation considered by
\citet{wang2018function} is that the tangent kernel $JJ^\top$, rather than
the Jacobian $J$ itself, remains approximately constant during training.
This makes the parameterisation-induced function-space geometry approximately
time-independent, but it does not imply that $JJ^\top$ is the identity or
that $JJ^\top\bm{\phi}$ is aligned with $\bm{\phi}$.
To investigate, for each recorded iteration, we formed the block-diagonal Jacobian and the stacked functional direction
\begin{equation*}
    J^t
    =
    \operatorname{diag}
    \left(
    J_1^t,\ldots,J_m^t
    \right),
    \qquad
    \bm{\phi}^t
    =
    \operatorname{col}
    \left(
    \bm{\phi}_1^t,\ldots,\bm{\phi}_m^t
    \right)
\end{equation*}
and calculated the cosine similarity alignment diagnostic  
\begin{equation}
    \frac{
    \left\langle
    J^t(J^t)^\top\bm{\phi}^t,
    \bm{\phi}^t
    \right\rangle
    }{
    \left\|
    J^t(J^t)^\top\bm{\phi}^t
    \right\|_2
    \left\|
    \bm{\phi}^t
    \right\|_2
    } , \label{eq: cosine similarity}
\end{equation}
with results shown in
\Cref{fig:fvgd_diagnosis_pullback_alignment}.
After the initial transient, the cosine similarity remains close to zero in all six diagnostic configurations. The mean values over the recorded iterations with $t \geq 1000$ range from approximately $0.002$ to $0.012$. This shows that, in these diagnostic runs, the linearised functional change induced by the raw Jacobian pullback is poorly aligned with the intended Euclidean function-space direction.
For the NTK-repulsion diagnostic in \Cref{fig:fvgd_diagnosis_ntk_repulsion_failure}, we instead compute the regularised inverse pullback
\begin{equation*}
    \bm{\beta}_{\mathrm{rep}}^t
    =
    \left(
    (J^t)^\top J^t+\rho I
    \right)^{-1}
    (J^t)^\top\bm{\phi}_{\mathrm{rep}}^t,
    \qquad
    \rho=10^{-3}.
\end{equation*}
The inverse pullback increases the mean repulsive-direction cosine similarity to approximately $0.875$. However, the mean parameter-space norm is only approximately $7.1\times10^{-7}$, compared with a mean supervised-gradient norm of approximately $1.14\times10^{-1}$ in the same run. Thus, improving the alignment alone does not make the repulsive update sufficiently large to compete with the supervised update.
These experiments are supplementary mechanistic diagnostics using the same narrow Spiral architecture as the main experiment, but with an augmented projection set and an RBF kernel on log probabilities.